%% file: journal2012.tex
\documentclass[journal]{IEEEtran}

\usepackage{amssymb}
\usepackage{graphicx}
\usepackage{times}
\usepackage{subfigure}
\usepackage{color}
\usepackage{amsbsy}
\usepackage{amsmath}
\usepackage{amsfonts}
\usepackage{amssymb}
\usepackage{epsfig}
\usepackage{tabulary}
\usepackage{wrapfig}
\usepackage[noadjust]{cite}

\definecolor{rltred}{rgb}{0.5,0,0}
\definecolor{rltgreen}{rgb}{0,0.5,0}
\definecolor{rltblue}{rgb}{0,0,0.5}

\usepackage{url}

\begin{document}


\title{Second-Order Non-Stationary
  Online Learning \\for Regression}


%
%

\author{Nina~Vaits, 
        Edward~Moroshko,
        and~Koby~Crammer,~\IEEEmembership{Fellow,~IEEE}
\thanks{All authors are with the Department
of Electrical Engineering, The Technion - Israel Institute of
Technology, Haifa 32000, Israel}
\thanks{Vaits and Moroshko have equal contribution to the paper.}
\thanks{Manuscript received Month date, 2013; revised month date,
  2013.}
\thanks{grant}}


%
%


\maketitle
\input{prela}

\input{abstract}
\input{intro}
\input{problem_setting}

\input{ARCOR}
\input{LASER}
\input{discussion}

\input{related}
\input{analysis_ARCOR}
\input{analysis_LASER}

\input{simulations1}

\input{summary}
\appendix
\input{appendix}

{
\bibliography{bib}
\bibliographystyle{plain}}
\end{document}

%% file: prela.tex
%


\newtheorem{theorem}{Theorem}
\newtheorem{lemma}[theorem]{Lemma}
\newtheorem{corollary}[theorem]{Corollary}
\def\proofsketch{\par\penalty-1000\vskip .5 pt\noindent{\bf Proof sketch\/: }}
\def\ProofSketch{\par\penalty-1000\vskip .1 pt\noindent{\bf Proof sketch\/: }}

\newcommand{\todo}[1]{{~\\\bf TODO: {#1}}~\\}

\newfont{\msym}{msbm10}
\newcommand{\reals}{\mathbb{R}}
\newcommand{\half}{\frac{1}{2}}
\newcommand{\sign}{{\rm sign}}
\newcommand{\paren}[1]{\left({#1}\right)}
\newcommand{\brackets}[1]{\left[{#1}\right]}
\newcommand{\braces}[1]{\left\{{#1}\right\}}
\newcommand{\ceiling}[1]{\left\lceil{#1}\right\rceil}
\newcommand{\abs}[1]{\left\vert{#1}\right\vert}
\newcommand{\tr}{{\rm Tr}}
\newcommand{\pr}[1]{{\rm Pr}\left[{#1}\right]}
\newcommand{\prp}[2]{{\rm Pr}_{#1}\left[{#2}\right]}
\newcommand{\Exp}[1]{{\rm E}\left[{#1}\right]}
\newcommand{\Expp}[2]{{\rm E}_{#1}\left[{#2}\right]}
\newcommand{\eqdef}{\stackrel{\rm def}{=}}
\newcommand{\comdots}{, \ldots ,}
\newcommand{\true}{\texttt{True}}
\newcommand{\false}{\texttt{False}}
\newcommand{\mcal}[1]{{\mathcal{#1}}}
\newcommand{\argmin}[1]{\underset{#1}{\mathrm{argmin}} \:}
\newcommand{\normt}[1]{\left\Vert {#1} \right\Vert^2}
\newcommand{\step}[1]{\left[#1\right]_+}
\newcommand{\1}[1]{[\![{#1}]\!]}
\newcommand{\diag}{{\textrm{diag}}}
\newcommand{\KL}{{\textrm{D}_{\textrm{KL}}}}
\newcommand{\IS}{{\textrm{D}_{\textrm{IS}}}}
\newcommand{\EU}{{\textrm{D}_{\textrm{EU}}}}

\newcommand{\leftmarginpar}[1]{\marginpar[#1]{}}
\newcommand{\figline}{\rule{0.50\textwidth}{0.5pt}}
\newcommand{\pseudocodefont}{\normalsize}
\newcommand{\nolineskips}{
\setlength{\parskip}{0pt}
\setlength{\parsep}{0pt}
\setlength{\topsep}{0pt}
\setlength{\partopsep}{0pt}
\setlength{\itemsep}{0pt}}

\newcommand{\beq}[1]{\begin{equation}\label{#1}}
\newcommand{\eeq}{\end{equation}}
\newcommand{\beqa}{\begin{eqnarray}}
\newcommand{\eeqa}{\end{eqnarray}}
\newcommand{\secref}[1]{Sec.~\ref{#1}}
\newcommand{\figref}[1]{Fig.~\ref{#1}}
\newcommand{\exmref}[1]{Example~\ref{#1}}
\newcommand{\thmref}[1]{Theorem~\ref{#1}}
\newcommand{\sthmref}[1]{Thm.~\ref{#1}}
\newcommand{\defref}[1]{Definition~\ref{#1}}
\newcommand{\remref}[1]{Remark~\ref{#1}}
\newcommand{\chapref}[1]{Chapter~\ref{#1}}
\newcommand{\appref}[1]{Appendix~\ref{#1}}
\newcommand{\lemref}[1]{Lemma~\ref{#1}}
\newcommand{\propref}[1]{Proposition~\ref{#1}}
\newcommand{\claimref}[1]{Claim~\ref{#1}}

\newcommand{\corref}[1]{Corollary~\ref{#1}}
\newcommand{\scorref}[1]{Cor.~\ref{#1}}
\newcommand{\tabref}[1]{Table~\ref{#1}}
\newcommand{\tran}[1]{{#1}^{\top}}
\newcommand{\norm}{\mcal{N}}
\newcommand{\eqsref}[1]{Eqns.~(\ref{#1})}

\newcommand{\mb}[1]{{\boldsymbol{#1}}}
\newcommand{\up}[2]{{#1}^{#2}}
\newcommand{\dn}[2]{{#1}_{#2}}
\newcommand{\du}[3]{{#1}_{#2}^{#3}}
\renewcommand{\star}[1]{\up{#1}{*}}
\newcommand{\textl}[2]{{$\textrm{#1}_{\textrm{#2}}$}}

\newcommand{\vx}{\mb{x}}
\newcommand{\vxi}[1]{\vx_{#1}}
\newcommand{\vxii}{\vxi{t}}

\newcommand{\ve}{\mb{e}}
\newcommand{\vei}[1]{\ve_{#1}}
\newcommand{\veii}{\vei{t}}
\newcommand{\vet}{\ve^{\top}}
\newcommand{\veti}[1]{\vet_{#1}}
\newcommand{\vetii}{\veti{i}}

\newcommand{\yi}[1]{y_{#1}}
\newcommand{\yii}{\yi{t}}
\newcommand{\hyi}[1]{\hat{y}_{#1}}
\newcommand{\hyii}{\hyi{i}}

\newcommand{\vy}{\mb{y}}
\newcommand{\vyi}[1]{\vy_{#1}}
\newcommand{\vyii}{\vyi{i}}

\newcommand{\vn}{\mb{\nu}}
\newcommand{\vni}[1]{\vn_{#1}}
\newcommand{\vnii}{\vni{i}}

\newcommand{\vmu}{\mb{\mu}}
\newcommand{\vmus}{{\vmu^*}}
\newcommand{\vmuts}{{\vmus}^{\top}}
\newcommand{\vmui}[1]{\vmu_{#1}}
\newcommand{\vmuii}{\vmui{i}}

\newcommand{\vmut}{\vmu^{\top}}
\newcommand{\vmuti}[1]{\vmut_{#1}}
\newcommand{\vmutii}{\vmuti{i}}

\newcommand{\vdelta}{\mb{\delta}}
\newcommand{\vdeltat}{\tran{\vdelta}}
\newcommand{\vdeltai}[1]{\vdelta_{#1}}

\newcommand{\vsigma}{\mb \sigma}
\newcommand{\msigma}{\Sigma}
\newcommand{\msigmas}{{\msigma^*}}
\newcommand{\msigmai}[1]{\msigma_{#1}}
\newcommand{\msigmaii}{\msigmai{t}}

\newcommand{\mups}{\Upsilon}
\newcommand{\mupss}{{\mups^*}}
\newcommand{\mupsi}[1]{\mups_{#1}}
\newcommand{\mupsii}{\mupsi{i}}
\newcommand{\upssl}{\upsilon^*_l}

\newcommand{\vu}{\mb{u}}
\newcommand{\vut}{\tran{\vu}}
\newcommand{\vui}[1]{\vu_{#1}}
\newcommand{\vuti}[1]{\vut_{#1}}
\newcommand{\hvu}{\hat{\vu}}
\newcommand{\hvut}{\tran{\hvu}}
\newcommand{\hvur}[1]{\hvu_{#1}}
\newcommand{\hvutr}[1]{\hvut_{#1}}
\newcommand{\vw}{\mb{w}}
\newcommand{\vwi}[1]{\vw_{#1}}
\newcommand{\vwii}{\vwi{t}}
\newcommand{\vwti}[1]{\vwt_{#1}}
\newcommand{\vwt}{\tran{\vw}}

\newcommand{\tvw}{\tilde{\mb{w}}}
\newcommand{\tvwi}[1]{\tvw_{#1}}
\newcommand{\tvwii}{\tvwi{t}}

\newcommand{\vv}{\mb{v}}
\newcommand{\vvt}{\tran{\vv}}

\newcommand{\vvi}[1]{\vv_{#1}}
\newcommand{\vvti}[1]{\vvt_{#1}}
\newcommand{\lambdai}[1]{\lambda_{#1}}
\newcommand{\Lambdai}[1]{\Lambda_{#1}}

\newcommand{\vxt}{\tran{\vx}}
\newcommand{\hvx}{\hat{\vx}}
\newcommand{\hvxi}[1]{\hvx_{#1}}
\newcommand{\hvxii}{\hvxi{i}}
\newcommand{\hvxt}{\tran{\hvx}}
\newcommand{\hvxti}[1]{\hvxt_{#1}}
\newcommand{\hvxtii}{\hvxti{i}}
\newcommand{\vxti}[1]{\vxt_{#1}}
\newcommand{\vxtii}{\vxti{i}}

\newcommand{\vb}{\mb{b}}
\newcommand{\vbt}{\tran{\vb}}
\newcommand{\vbi}[1]{\vb_{#1}}

\newcommand{\hvy}{\hat{\vy}}
\newcommand{\hvyi}[1]{\hvy_{#1}}


\renewcommand{\mp}{P}
\newcommand{\mpd}{\mp^{(d)}}
\newcommand{\mpt}{\mp^T}
\newcommand{\tmp}{\tilde{\mp}}
\newcommand{\mpi}[1]{\mp_{#1}}
\newcommand{\mpti}[1]{\mpt_{#1}}
\newcommand{\mptii}{\mpti{i}}
\newcommand{\mpii}{\mpi{i}}
\newcommand{\mps}{Q}
\newcommand{\mpsi}[1]{\mps_{#1}}
\newcommand{\mpsii}{\mpsi{i}}
\newcommand{\tmpt}{\tmp^T}
\newcommand{\mz}{Z}
\newcommand{\mv}{V}
\newcommand{\mvi}[1]{\mv_{#1}}
\newcommand{\mvt}{V^T}
\newcommand{\mvti}[1]{\mvt_{#1}}
\newcommand{\mzt}{\mz^T}
\newcommand{\tmz}{\tilde{\mz}}
\newcommand{\tmzt}{\tmz^T}
\newcommand{\mx}{\mathbf{X}}
\newcommand{\ma}{\mathbf{A}}
\newcommand{\mxs}[1]{\mx_{#1}}

\newcommand{\mai}[1]{\ma_{#1}}
\newcommand{\mat}{\tran{\ma}}
\newcommand{\mati}[1]{\mat_{#1}}

\newcommand{\mc}{{C}}
\newcommand{\mci}[1]{\mc_{#1}}
\newcommand{\mcti}[1]{\mct_{#1}}

\newcommand{\md}{{\mathbf{D}}}
\newcommand{\mdi}[1]{\md_{#1}}

\newcommand{\mxi}[1]{\textrm{diag}^2\paren{\vxi{#1}}}
\newcommand{\mxii}{\mxi{i}}

\newcommand{\hmx}{\hat{\mx}}
\newcommand{\hmxi}[1]{\hmx_{#1}}
\newcommand{\hmxii}{\hmxi{i}}
\newcommand{\hmxt}{\hmx^T}
\newcommand{\mxt}{\mx^\top}
\newcommand{\mi}{I}
\newcommand{\mq}{Q}
\newcommand{\mqt}{\mq^T}
\newcommand{\mlam}{\Lambda}

\renewcommand{\L}{\mcal{L}}
\newcommand{\R}{\mcal{R}}
\newcommand{\X}{\mcal{X}}
\newcommand{\Y}{\mcal{Y}}
\newcommand{\F}{\mcal{F}}
\newcommand{\nur}[1]{\nu_{#1}}
\newcommand{\lambdar}[1]{\lambda_{#1}}
\newcommand{\gammai}[1]{\gamma_{#1}}
\newcommand{\gammaii}{\gammai{i}}
\newcommand{\alphai}[1]{\alpha_{#1}}
\newcommand{\alphaii}{\alphai{i}}
\newcommand{\lossp}[1]{\ell_{#1}}
\newcommand{\eps}{\epsilon}
\newcommand{\epss}{\eps^*}
\newcommand{\lsep}{\lossp{\eps}}
\newcommand{\lseps}{\lossp{\epss}}
\newcommand{\T}{\mcal{T}}

\newcommand{\kc}[1]{\begin{center}\fbox{\parbox{3in}{{\textcolor{green}{KC: #1}}}}\end{center}}
\newcommand{\nv}[1]{\begin{center}\fbox{\parbox{3in}{{\textcolor{blue}{NV: #1}}}}\end{center}}
\newcommand{\edward}[1]{\begin{center}\fbox{\parbox{3in}{{\textcolor{red}{EM: #1}}}}\end{center}}

\newcommand{\newstuffa}[2]{#2}
\newcommand{\newstufffroma}[1]{}
\newcommand{\newstufftoa}{}

\newcommand{\newstuff}[2]{#2}
\newcommand{\newstufffrom}[1]{}
\newcommand{\newstuffto}{}
\newcommand{\oldnote}[2]{}

\newcommand{\commentout}[1]{}
\newcommand{\mypar}[1]{\medskip\noindent{\bf #1}}

\newcommand{\inner}[2]{\left< {#1} , {#2} \right>}
\newcommand{\kernel}[2]{K\left({#1},{#2} \right)}
\newcommand{\tprr}{\tilde{p}_{rr}}
\newcommand{\hxr}{\hat{x}_{r}}
\newcommand{\projalg}{{PST }}
\newcommand{\projealg}[1]{$\textrm{PST}_{#1}~$}
\newcommand{\gradalg}{{GST }}

\newcounter {mySubCounter}
\newcommand {\twocoleqn}[4]{
  \setcounter {mySubCounter}{0} %
  \let\OldTheEquation \theequation %
  \renewcommand {\theequation }{\OldTheEquation \alph {mySubCounter}}%
  \noindent \hfill%
  \begin{minipage}{.40\textwidth}
\vspace{-0.6cm}
    \begin{equation}\refstepcounter{mySubCounter}
      #1
    \end {equation}
  \end {minipage}
~~~~~~
  \addtocounter {equation}{ -1}%
  \begin{minipage}{.40\textwidth}
\vspace{-0.6cm}
    \begin{equation}\refstepcounter{mySubCounter}
      #3
    \end{equation}
  \end{minipage}%
  \let\theequation\OldTheEquation
}

\newcommand{\vzero}{\mb{0}}

\newcommand{\smargin}{\mcal{M}}

\newcommand{\ai}[1]{A_{#1}}
\newcommand{\bi}[1]{B_{#1}}
\newcommand{\aii}{\ai{i}}
\newcommand{\bii}{\bi{i}}
\newcommand{\betai}[1]{\beta_{#1}}
\newcommand{\betaii}{\betai{i}}
\newcommand{\mar}{M}
\newcommand{\mari}[1]{\mar_{#1}}
\newcommand{\marii}{\mari{i}}
\newcommand{\nmari}[1]{m_{#1}}
\newcommand{\nmarii}{\nmari{i}}

\newcommand{\erf}{\Phi}

\newcommand{\var}{V}
\newcommand{\vari}[1]{\var_{#1}}
\newcommand{\varii}{\vari{i}}

\newcommand{\varb}{v}
\newcommand{\varbi}[1]{\varb_{#1}}
\newcommand{\varbii}{\varbi{i}}

\newcommand{\vara}{u}
\newcommand{\varai}[1]{\vara_{#1}}
\newcommand{\varaii}{\varai{i}}

\newcommand{\marb}{m}
\newcommand{\marbi}[1]{\marb_{#1}}
\newcommand{\marbii}{\marbi{i}}

\newcommand{\algname}{{AROW}}
\newcommand{\rlsname}{{RLS}}
\newcommand{\mrlsname}{{MRLS}}

\newcommand{\phia}{\psi}
\newcommand{\phib}{\xi}

\newcommand{\amsigmaii}{\tilde{\msigma}_t}
\newcommand{\amsigmai}[1]{\tilde{\msigma}_{#1}}
\newcommand{\avmuii}{\tilde{\vmu}_i}
\newcommand{\avmui}[1]{\tilde{\vmu}_{#1}}
\newcommand{\amarbii}{\tilde{\marb}_i}
\newcommand{\avarbii}{\tilde{\varb}_i}
\newcommand{\avaraii}{\tilde{\vara}_i}
\newcommand{\aalphaii}{\tilde{\alpha}_i}

\newcommand{\svar}{v}
\newcommand{\smar}{m}
\newcommand{\nsmar}{\bar{m}}

\newcommand{\vnu}{\mb{\nu}}
\newcommand{\vnut}{\vnu^\top}
\newcommand{\vz}{\mb{z}}
\newcommand{\vZ}{\mb{Z}}
\newcommand{\fphi}{f_{\phi}}
\newcommand{\gphi}{g_{\phi}}


\newcommand{\vtmui}[1]{\tilde{\vmu}_{#1}}
\newcommand{\vtmuii}{\vtmui{i}}

\newcommand{\zetai}[1]{\zeta_{#1}}
\newcommand{\zetaii}{\zetai{i}}


\newcommand{\vstate}{\bf{s}}
\newcommand{\vstatet}[1]{\vstate_{#1}}
\newcommand{\vstatett}{\vstatet{t}}

\newcommand{\mtran}{\bf{\Phi}}
\newcommand{\mtrant}[1]{\mtran_{#1}}
\newcommand{\mtrantt}{\mtrant{t}}

\newcommand{\vstatenoise}{\bf{\eta}}
\newcommand{\vstatenoiset}[1]{\vstatenoise_{#1}}
\newcommand{\vstatenoisett}{\vstatenoiset{t}}

\newcommand{\vobser}{\bf{o}}
\newcommand{\vobsert}[1]{\vobser_{#1}}
\newcommand{\vobsertt}{\vobsert{t}}

\newcommand{\mobser}{\bf{H}}
\newcommand{\mobsert}[1]{\mobser_{#1}}
\newcommand{\mobsertt}{\mobsert{t}}

\newcommand{\vobsernoise}{\bf{\nu}}
\newcommand{\vobsernoiset}[1]{\vobsernoise_{#1}}
\newcommand{\vobsernoisett}{\vobsernoiset{t}}

\newcommand{\mstatenoisecov}{\bf{Q}}
\newcommand{\mstatenoisecovt}[1]{\mstatenoisecov_{#1}}
\newcommand{\mstatenoisecovtt}{\mstatenoisecovt{t}}

\newcommand{\mobsernoisecov}{\bf{R}}
\newcommand{\mobsernoisecovt}[1]{\mobsernoisecov_{#1}}
\newcommand{\mobsernoisecovtt}{\mobsernoisecovt{t}}

\newcommand{\vestate}{\bf{\hat{s}}}
\newcommand{\vestatet}[1]{\vestate_{#1}}
\newcommand{\vestatett}{\vestatet{t}}
\newcommand{\vestatept}[1]{\vestatet{#1}^+}
\newcommand{\vestatent}[1]{\vestatet{#1}^-}

\newcommand{\mcovar}{\bf{P}}
\newcommand{\mcovart}[1]{\mcovar_{#1}}
\newcommand{\mcovarpt}[1]{\mcovart{#1}^+}
\newcommand{\mcovarnt}[1]{\mcovart{#1}^-}

\newcommand{\mkalmangain}{\bf{K}}
\newcommand{\mkalmangaint}[1]{\mkalmangain_{#1}}

\newcommand{\vkalmangain}{\bf{\kappa}}
\newcommand{\vkalmangaint}[1]{\vkalmangain_{#1}}

\newcommand{\obsernoise}{{\nu}}
\newcommand{\obsernoiset}[1]{\obsernoise_{#1}}
\newcommand{\obsernoisett}{\obsernoiset{t}}

\newcommand{\obsernoisecov}{r}
\newcommand{\obsernoisecovt}[1]{\obsernoisecov_{#1}}
\newcommand{\obsernoisecovtt}{\obsernoisecov}

\newcommand{\obsnscv}{s}
\newcommand{\obsnscvt}[1]{\obsnscv_{#1}}
\newcommand{\obsnscvtt}{\obsnscvt{t}}

\newcommand{\Psit}[1]{\Psi_{#1}}
\newcommand{\Psitt}{\Psit{t}}

\newcommand{\Omegat}[1]{\Omega_{#1}}
\newcommand{\Omegatt}{\Omegat{t}}

\newcommand{\ellt}[1]{\ell_{#1}}
\newcommand{\gllt}[1]{g_{#1}}

\newcommand{\chit}[1]{\chi_{#1}}

\newcommand{\ms}{\mathcal{M}}
\newcommand{\us}{\mathcal{U}}
\newcommand{\as}{\mathcal{A}}

\newcommand{\mn}{M}
\newcommand{\un}{U}

\newcommand{\set}{S}
\newcommand{\seti}[1]{S_{#1}}

\newcommand{\obj}{\mcal{C}}

\newcommand{\dta}[3]{d_{#3}\paren{#1,#2}}

\newcommand{\coa}{a}
\newcommand{\coc}{c}
\newcommand{\cob}{b}
\newcommand{\cor}{r}
\newcommand{\conu}{\nu}

\newcommand{\coat}[1]{\coa_{#1}}
\newcommand{\coct}[1]{\coc_{#1}}
\newcommand{\cobt}[1]{\cob_{#1}}
\newcommand{\cort}[1]{\cor_{#1}}
\newcommand{\conut}[1]{\conu_{#1}}

\newcommand{\coatt}{\coat{t}}
\newcommand{\coctt}{\coct{t}}
\newcommand{\cobtt}{\cobt{t}}
\newcommand{\cortt}{\cort{t}}
\newcommand{\conutt}{\conut{t}}

\newcommand{\rb}{R_B}
\newcommand{\proj}{\textrm{proj}}

%% file: abstract.tex
\begin{abstract}
  The goal of a learner, in standard online learning, is to have the
  cumulative loss not much larger compared with the best-performing
  function from some fixed class. Numerous algorithms were shown to
  have this gap arbitrarily close to zero, compared with the best
  function that is chosen off-line. Nevertheless, many real-world
  applications, such as adaptive filtering, are non-stationary in
  nature, and the best prediction function may drift over time. We
  introduce two novel algorithms for online regression, designed to work well in non-stationary environment. Our first algorithm performs adaptive resets to forget the history,  
  while the second is last-step min-max optimal in context of a drift.
  We analyze both algorithms in the worst-case regret framework and
  show that they maintain an average loss close to that of the best
  slowly changing sequence of linear functions, as long as the
  cumulative drift is sublinear.  In addition, in the stationary case,
  when no drift occurs, our algorithms suffer logarithmic regret, as
  for previous algorithms.  Our bounds improve over the existing ones,
  and simulations demonstrate the usefulness of these algorithms
  compared with other state-of-the-art approaches.
\end{abstract}

%% file: intro.tex
\section{Introduction}
We consider the classical problem of online learning for
regression. On each iteration, an algorithm receives a new instance
(for example, input from an array of antennas) and outputs a
prediction of a real value (for example distance to the source). The
correct value is then revealed, and the algorithm suffers a loss based
on both its prediction and the correct output value.

In the past half a century many algorithms were proposed~(see e.g. a
comprehensive book~\cite{CesaBiGa06}) for this problem, some of which
are able to achieve an average loss arbitrarily close to that of the
best function in retrospect. Furthermore, such guarantees hold even if
the input and output pairs are chosen in a fully adversarial manner
with no distributional assumptions. Many of these algorithms exploit
first-order information (e.g. gradients).

Recently there is an increased amount of interest in algorithms that
exploit second order information. For example the second order
perceptron algorithm~\cite{CesaBianchiCoGe05}, confidence-weighted
learning~\cite{DredzeCrPe08,CrammerDrPe08}, adaptive regularization of
weights (AROW)~\cite{CrammerKuDr09}, all designed for classification;
and AdaGrad~\cite{DuchiHS10} and FTPRL~\cite{McMahanS10}  for
general loss functions.

Despite the extensive and impressive guarantees that can be made for
algorithms in such settings, competing with the best {\em fixed}
function is not always good enough. In many real-world applications,
the true target function is not {\em fixed}, but is {\em slowly} changing
over time. Consider a filter designed to cancel echoes in a hall. Over
time, people enter and leave the hall, furniture are being moved,
microphones are replaced and so on. When this drift occurs, the
predictor itself must also change in order to remain relevant.

With such properties in mind, we develop new learning algorithms,
based on second-order quantities, designed to work with target
drift. The goal of an algorithm is to maintain an average loss close
to that of the best slowly changing sequence of functions, rather than
compete well with a single function. We focus on problems for which
this sequence consists only of linear functions. Most previous
algorithms~(e.g. \cite{LittlestoneW94,ECCC-TR00-070,HerbsterW01,KivinenSW01})
designed for this problem are based on first-order information, such
as gradient descent, with additional control on the norm of the weight-vector used for
prediction~\cite{KivinenSW01} or the number of inputs used to define
it~\cite{CavallantiCG07}.

In \secref{sec:problem_setting} we
review three second-order learning algorithms: the recursive least
squares (RLS)~\cite{Hayes} algorithm, the Aggregating Algorithm for
regression (AAR)~(\cite{vovkAS,Vovk01}), which can be shown to be
derived based on a last-step min-max approach~\cite{Forster}, and the AROWR
algorithm~\cite{VaitsCr11} which is a modification of the AROW
algorithm~\cite{CrammerKuDr09} for regression.  All three algorithms
obtain logarithmic regret in the stationary setting, although derived
using different approaches, and they are not equivalent in general.
In \secref{sec:non_stat} we formally present the non-stationary
setting both in terms of algorithms and in terms of theoretical
analysis.

For the RLS algorithm, a variant called CR-RLS~(\cite{Salgado,Goodwin,Chen}) for
the non-stationary setting was described, yet not analyzed, before.
In \secref{sec:algorithms} we present two algorithms for the
non-stationary setting, that build on the other two
algorithms. Specifically, in \secref{sec:ARCOR} we extend the AROWR
algorithm for the non-stationary setting, yielding an algorithm
called ARCOR for adaptive regularization with covariance reset. Similar to CR-RLS, ARCOR performs a step called covariance-reset,
which resets the second-order information from time-to-time, yet it is
done based on the properties of this covariance-like matrix, and not
based on the number of examples observed, as in CR-RLS.

In \secref{sec:LASER} we derive different algorithm based on the
last-step min-max approach proposed by Forster~\cite{Forster} and later
used~\cite{TakimotoW00} for online density estimation. On each
iteration the algorithm makes the optimal min-max prediction with
respect to the regret, assuming it is the last iteration. Yet, unlike
previous work~\cite{Forster}, it is optimal when a drift is
allowed. As opposed to the derivation of the last-step min-max
predictor for a fixed vector, the resulting optimization problem is
not straightforward to solve.  We develop a dynamic program (a
recursion) to solve this problem, which allows to compute the optimal
last-step min-max predictor. We call this algorithm LASER for last
step adaptive regressor algorithm. We conclude the algorithmic part in
\secref{sec:discussion} in which we compare all non-stationary
algorithms head-to-head highlighting their similarities and
differences. Additionally, after describing the details of our
algorithms, we provide in \secref{sec:related_work} a comprehensive
review of previous work, that puts our contribution in
perspective. Both algorithms reduce to their stationary counterparts
when no drift occurs.

We then move to \secref{sec:regret} which summarizes our next
contribution stating and proving regret bounds for both algorithms.
We analyse both algorithms in the worst-case regret-setting and show
that as long as the amount of average-drift is sublinear, the
average-loss of both algorithms will converge to the average-loss of
the best sequence of functions. Specifically, we show in
\secref{sec:analysis_arcor} that the cumulative loss of ARCOR after
observing $T$ examples, denoted by $ L_T(\textrm{ARCOR})$, is upper
bounded by the cumulative loss of any {\em sequence} of weight-vectors
$\{ \vui{t}\}$, denoted by $L_T( \{ \vui{t} \} )$, plus an additional
term $\mcal{O} \paren{ T^{1/2}\paren{ V( \{ \vui{t} \} ) }^{1/2}\log
  T}$ where $V( \{ \vui{t} \} ) $ measures the differences (or
variance) between consecutive weight-vectors of the sequence $\{
\vui{t} \}$.  Later, we show in \secref{sec:analysis_laser} a similar
bound for the loss of LASER, denoted by $L_T(\textrm{LASER})$, for which the second
term is $\mcal{O} \paren{ T^{2/3}\paren{ V( \{ \vui{t} \}
    )}^{1/3}}$. We emphasize that in both bounds the measure $V( \{
\vui{t} \} )$ of differences between consecutive weight-vectors is not
defined in the same way, and thus, the bounds are not comparable in general.

In \secref{sec:simulations} we report results of simulations designed
to highlight the properties of both algorithms, as well as the
commonalities and differences between them. We conclude in
\secref{sec:summary_conclusions} and most of the technical proofs
appear in the appendix.

The ARCOR algorithm was presented in a shorter
publication~\cite{VaitsCr11}, as well with its analysis and some of
its details. The LASER algorithm and its analysis was also presented
in a shorter version~\cite{MoroshkoCr13}. The contribution of this submission is
three-fold. First, we provide head-to-head comparison of three
second-order algorithms for the stationary case. Second, we fill the
gap of second-order algorithms for the non-stationary
case. Specifically, we add to the CR-RLS (which extends RLS) and
design second-order algorithms for the non-stationary case and analyze
them, building both on AROWR and AAR. Our algorithms are derived from
different principles from each other, which is reflected in our
analysis. Finally, we provide empirical evidence showing that under
various conditions different algorithm performs the best.

Some notation we use throughout the paper:
For a symmetric matrix $\msigma$ we denote its j$th$ eigenvalue by
$\lambda_j(\msigma)$. Similarly we denote its smallest eigenvalue by
$\lambda_{min}(\msigma) = \min_j \lambda_j(\msigma)$, and its largest
eigenvalue by $\lambda_{max}(\msigma) = \max_j \lambda_j(\msigma)$.
For a vector $\vu\in\reals^d$, we denote by $\Vert\vu \Vert$ the $\ell_2$-norm of the vector.
Finally, for  $y>0$ we define $clip(x,y)=\sign(x) \min\{\vert x \vert,
y\}$.

%% file: problem_setting.tex
\section{Stationery Online Learning}
\label{sec:problem_setting}
We focus on the regression task evaluated with the
squared loss.  Our algorithms are designed for the online setting and
work in iterations (or rounds). On each round an online algorithm
receives an input-vector $\vxi{t}\in\reals^d$ and predicts a real
value $\hyi{t}\in\reals$. Then the algorithm receives a target label
$\yi{t}\in\reals$ associated with $\vxi{t}$, and uses it to update its
prediction rule, and proceeds to the next round.

At each iteration, the performance of the algorithm is evaluated using the
squared loss, $\ell_t(\textrm{alg})=\ell\paren{ \yi{t}, \hyi{t} }
= \paren{\hyi{t}- \yi{t} }^2$. The cumulative loss suffered  by the
 algorithm
over $T$ iterations is, \(
L_{T}(\textrm{alg})=\sum_{t=1}^{T}\ell_{t}(\textrm{alg})
. 
\) 

The goal of the algorithm is to have low cumulative loss compared to
predictors from some class. A large body of work, which we adopt as
well, is focused on linear prediction functions of the form
$f(\vx)=\vxt\vu$ where $\vu\in\reals^d$ is some weight-vector. We
denote by $\ell_t(\vu) = \paren{\vxti{t}\vu-\yi{t}}^2$ the
instantaneous loss of a weight-vector $\vu$. The cumulative loss
suffered by a fixed weight-vector $\vu$ is, $L_T(\vu) = \sum_t^T
\ell_t(\vu)$.

The goal of the learning algorithm is
to suffer low loss compared with the best linear function. Formally we
define the regret of an algorithm to be
\begin{align}
{R}(T) = L_T(\textrm{alg}) - \inf_\vu L_T(\vu) ~.
\label{stationary_regret}
\end{align}

The goal of an algorithm is to have ${R}(T) = o(T)$, such that
the average loss will converge to the average loss of the best linear
function $\vu$.

Numerous algorithms were developed for this problem, see a
comprehensive review in the book of Cesa-Bianchi and Lugosi~\cite{CesaBiGa06}.
Among these, a few second-order online algorithms for regression were
proposed in recent years, which we summarize in
\tabref{table:stationary_algorithms}.
One approach for online learning is to reduce the problem into
consecutive batch problems, and specifically use all previous examples
to generate a classifier, which is used to predict current example.
Recursive least squares (RLS)~\cite{Hayes} approach, for example, sets
a weight-vector to be the solution of the following optimization
problem,
\[
\vwi{t} = \arg\min_\vw \paren{\sum_{i=1}^{t} r^{t-i}\paren{\yi{i}-\vw\cdot\vxi{i}}^2}
\]
Since the last problem grows with time, the well known recursive least
squares (RLS)~\cite{Hayes} algorithm was developed to generate a
solution recursively.  The RLS algorithm maintains both a vector
$\vwi{t}$ and a positive semi-definite (PSD) matrix $\msigmai{t}$. On
each iteration, after making a prediction $\hyi{t} =
\vxti{t}\vwi{t-1}$, the algorithm receives the true label $\yi{t}$ and
updates,
\begin{align}
\vwi{t}  &=
\vwi{t-1}+\frac{(\yi{t}-\vxti{t}\vwi{t-1})\msigmai{t-1}\vxi{t}}{\cor+\vxti{t}\msigmai{t-1}\vxi{t}}\label{rls_update_weights}\\
{\msigma}_{t}^{-1} &= \cor \msigmai{t-1}^{-1} +\vxi{t}\vxti{t}\label{rls_update_sigma}~.
\end{align}
The update of the prediction vector $\vwi{t}$ is additive, with vector
$\msigmai{t-1}\vxi{t}$ scaled by the error
$(\yi{t}-\vxti{t}\vwi{t-1})$ over the norm of the input measured using
the norm defined by the matrix $\vxti{t}\msigmai{t-1}\vxi{t}$. The
algorithm is summarized in the right column of
\tabref{table:stationary_algorithms}.

The Aggregating Algorithm for regression (AAR)~\cite{vovkAS,Vovk01},
summarized in the middle column of
\tabref{table:stationary_algorithms}, was introduced by Vovk and it is
similar to the RLS algorithm, except it shrinks its predictions. The
AAR algorithm was shown to be last-step min-max optimal by
Forster~\cite{Forster}. Given a new input $\vxi{T}$ the algorithm
predicts $\hyi{T}$ which is the minimizer of the following problem,
\begin{align}
\arg\min_{\hyi{T}} \max_{\yi{T}} \brackets{\sum_{t=1}^{T} (\yi{t} -
  \hyi{t})^2  - \inf_{\vu} \paren{b\left\Vert \vu\right\Vert
    ^{2}+L_{T}(\vu)}}.\!\!\!
\label{minmax_algorithm_1}
\end{align}
Forster proposed also a simpler analysis with the same regret bound.

Finally, the AROWR algorithm~\cite{VaitsCr11} is a modification of the
AROW algorithm~\cite{CrammerKuDr09} for regression. In a nutshell, the
AROW algorithm maintains a Gaussian distribution parameterized by a
mean $\vwi{t}\in\reals^d$ and a full covariance matrix
$\msigmaii\in\reals^{d \times d}$. Intuitively, the mean $\vwi{t}$
represents a current linear function, while the covariance matrix
$\msigmai{t}$ captures the uncertainty in the linear function
$\vwi{t}$. Given a new example $(\vxii,\yii)$ the algorithm uses its
current mean to make a prediction $\hyi{t} = \vxti{t}\vwi{t-1}$.
AROWR then sets the new distribution to be the solution of the
following optimization problem,
\begin{align}
 \arg\min_{\vw, \msigma} ~& \KL\paren{ \norm\paren{\vw, \msigma} \,\Vert\,
    \norm\paren{\vwi{t-1}, \msigmai{t-1}}} \nonumber\\
  ~&  + \frac{1}{2r}
  \paren{\yi{t} - \vwt\vxii}^2+ \frac{1}{2r} \paren{\vxti{t}
    \msigma \vxi{t}}
\label{arowr_objective}
\end{align}
This optimization problem is similar to the one of
AROW~\cite{CrammerKuDr09} for classification, except we use the square
loss rather than squared-hinge loss used in AROW.  Intuitively, the
optimization problem trades off between three requirements. The first
term forces the parameters not to change much per example, as the
entire learning history is encapsulated within them. The second term
requires that the new vector $\vwi{t}$ should perform well on the
current instance, and finally, the third term reflects the fact that
the uncertainty about the parameters reduces as we observe the current
example $\vxi{t}$.


The weight vector solving this optimization problem (details given by
Vaits and Crammer~\cite{VaitsCr11}) is given by,
\begin{align}
\vwi{t}  
& = \vwi{t-1} +\paren{ \frac{ \yi{t}- \vwi{t-1}\cdot\vxi{t}  }{\obsernoisecov + \vxti{t}\msigmai{t-1} \vxi{t}
  }  } \msigmai{t-1} \vxi{t}\label{update_rule_mean_rls}~,
\end{align}
and the optimal covariance matrix is,
%
\begin{align}
\msigmai{t}^{-1} = \msigmai{t-1}^{-1} +
\frac{1}{r} \vxi{t}\vxti{t} ~.
\label{update_rule_sigma}
\end{align}
The algorithm is summarized in the left column of
\tabref{table:stationary_algorithms}.  Comparing AROW to RLS we
observe that while the update of the weights of
\eqref{update_rule_mean_rls} is equivalent to the update of RLS in
\eqref{rls_update_weights}, the update of the matrix
\eqref{rls_update_sigma} for RLS is not equivalent to
\eqref{update_rule_sigma}, as in the former case the matrix goes via a
multiplicative update as well as additive, while in
\eqref{update_rule_sigma} the update is only additive. The two updates
are equivalent only by setting $\cor=1$.  Moving to AAR, we note that
the update rules for $\vwi{t}$ and $\msigmai{t}$ in AROWR and AAR are
the same if we define $\msigmai{t}^{AAR}=\msigmai{t}^{AROWR}/r$, but
AROWR does not shrink its predictions as AAR.  Thus all three
algorithms are not equivalent, although very similar.

\section{Non-Stationary Online Learning}
\label{sec:non_stat}
All previous algorithms assume both by design and analysis that the data is
stationary. The analysis of all algorithms compares their performance
to that of a single fixed weight vector $\vu$, and all suffer regret
that is logarithmic is $T$.

We use an extended notion of evaluation, comparing
our algorithms to a sequence of functions. We define the loss suffered
by such a sequence to be,
\[
L_T(\vui{1}\comdots\vui{T}) = L_T(\{\vui{t}\}) = \sum_t^T \ell_t(\vui{t})~,
\]
and the regret is then defined to be,
\begin{align}
{R}(T) = L_T(\textrm{alg}) - \inf_{\vui{1}\comdots\vui{T}}
L_T(\{\vui{t}\}) ~.
\label{nonstationry_regret}
\end{align}
We focus on algorithms that are able to compete against sequences of
weight-vectors, $(\vui{1} \comdots \vui{T})\in \reals^d \times \dots
\times \reals^d$, where $\vui{t}$ is used to make a prediction for the t$th$ example
$(\vxi{t},\yi{t})$.

Clearly, with no restriction over the set $\{\vui{t}\}$ the right term of
the regret can easily be zero by setting, $\vui{t} = \vxi{t}
(\yi{t}/\normt{\vxi{t}})$, which implies $\ell_t(\vui{t})=0$ for all
$t$.
Thus, in the analysis below we will make use of the total drift of the weight-vectors defined to be,
\[
V^{(P)} = V^{(P)}_T(\{\vui{t}\}) = \sum_{t=1}^{T-1} \Vert \vui{t}-\vui{t+1} \Vert ^P ~,
\]
where $P\in\{1,2\}$. 

For all three algorithms, as was also observed previously in the context
of CW~\cite{DredzeCrPe08}, AROW~\cite{CrammerKuDr09},
AdaGrad~\cite{DuchiHS10} and FTPRL~\cite{McMahanS10}, the matrix
$\msigma$ can be interpreted as adaptive learning rate. As these
algorithms process more examples, that is larger values of $t$, the
eigenvalues of the matrix ${\msigma}_{t}^{-1}$ increase, and the
eigenvalues of the matrix ${\msigma}_{t}$ decrease, and we get that
the rate of updates is getting smaller, since the additive term
$\msigmai{t-1}\vxi{t}$ is getting smaller.
As a
consequence the algorithms will gradually stop updating using current
instances which lie in the subspace of examples that were previously
observed numerous times. This property leads to a very fast
convergence in the stationary case.
However,
when we allow these algorithms to be compared with a sequence of
weight-vectors, each applied to a different input example, or
equivalently, there is a drift or shift of a good prediction vector,
these algorithms will perform poorly, as they will
converge and not be able to adapt to the non-stationarity nature of
the data.

This phenomena motivated the proposal of the CR-RLS
algorithm~(\cite{Salgado,Goodwin,Chen}), which re-sets the covariance
matrix every fixed number of input examples, causing the algorithm not
to converge or get stuck. The pseudo-code of CR-RLS algorithm is given
in the right column of \tabref{table:algorithms}. The only difference
of CR-RLS from RLS is that after updating the matrix $\msigmai{t}$, the
algorithm checks whether $T_0$ (a predefined natural number) examples
were observed since the last restart, and if this is the case, it sets
the matrix to be the identity matrix. Clearly, if $T_0=\infty$ the
CR-RLS algorithm is reduced to the RLS algorithm.


\begin{center}
\begin{table*}[ht]
{\small
\hfill{}
\caption{Algorithms for stationary setting} 
\label{table:stationary_algorithms}
\begin{tabulary}{\textwidth}{|C|C|L|L|L|} 
\hline                       
 &  & \bf AROWR & \bf AAR  & \bf RLS \\ [0.5ex] 
\hline                  
 Parameters  & & $0<\cor$ & $0<b$ & $0<\cor\leq1$ \\ 
\hline
 Initialize &  & $\vwi{0}=0 ~,~ \msigmai{0}=\mi$ &
 $\vwi{0}=0 ~,~ \msigmai{0}=b^{-1}\mi$ & $\vwi{0}=0 ~,~ \msigmai{0}=\mi$ \\ [0.5ex]
\hline
 & & \multicolumn{3}{|c|}{ Receive an instance $\vxi{t}$}  \\
\cline{2-5}
For & Output & \[\hyi{t}=\vxti{t}\vwi{t-1}\] & \[\hyi{t}=\frac{\vxti{t}\vwi{t-1}}{1+\vxti{t}\msigmai{t-1}\vxi{t}}\] & \[\hyi{t}=\vxti{t}\vwi{t-1}\] \\
$t=1 ... T$ & prediction & & &\\
\cline{2-5}
 & & \multicolumn{3}{|c|}{Receive a correct label $\yi{t}$ }  \\
 \cline{2-5}
 & Update $\msigmai{t}$: & \[\msigma_{t}^{-1} = \msigmai{t-1}^{-1} + \frac{1}{\cor}\vxi{t}\vxti{t}\] & \[\msigma_{t}^{-1} = \msigmai{t-1}^{-1}
+ \vxi{t}\vxti{t}\] & \[\msigma_{t}^{-1} = \cor\msigmai{t-1}^{-1} + \vxi{t}\vxti{t}\] \\
\cline{2-5}
 & Update $\vwii$: & \vspace{0.2cm}\[\begin{array}{ll}\vwi{t}  &\!\!\!\!=   \vwi{t-1}\\&\!\!\!\!+\frac{(\yi{t}-\vxti{t}\vwi{t-1})\msigmai{t-1}\vxi{t}}{\cor+\vxti{t}\msigmai{t-1}\vxi{t}}\\~\\ \end{array}\]
 & \[\!\!\!\!\!\!\begin{array}{ll}\vwi{t}  &\!\!\!\!\!= \vwi{t-1}\\&\!\!\!\!\!+\frac{(\yi{t}-\vxti{t}\vwi{t-1})\msigmai{t-1}\vxi{t}}{1+\vxti{t}\msigmai{t-1}\vxi{t}}\end{array}\]
& \[\begin{array}{ll}\vwi{t}  &\!\!\!\!= \vwi{t-1}\\&\!\!\!\!+\frac{(\yi{t}-\vxti{t}\vwi{t-1})\msigmai{t-1}\vxi{t}}{\cor+\vxti{t}\msigmai{t-1}\vxi{t}}\end{array}\] \\
 \hline
Output &  & $\vwi{T} \ ,\ \msigmai{T}$ & $\vwi{T} \ ,\ \msigmai{T}$ & $\vwi{T} \ ,\ \msigmai{T}$\\
 [1ex]     
\hline                  
 Extension to non-stationary setting  & & {\bf ARCOR} \secref{sec:ARCOR} below&
 {\bf LASER} \secref{sec:LASER} ~~~~ below&  {\bf CR-RLS}~\cite{Salgado,Goodwin,Chen} \\
\hline  
 Analysis  & & yes, \secref{sec:analysis_arcor} ~~~~~~~~~ below&
yes, \secref{sec:analysis_laser} ~~~~~~~~~ below&  No\\
\hline  
\end{tabulary}}
\hfill{}
\end{table*}
\end{center}

%% file: ARCOR.tex
\section{Algorithms for Non-Stationary Regression}
\label{sec:algorithms}
In this work we fill the gap and propose extension to non-stationary
setting for the two other algorithms in
\tabref{table:stationary_algorithms}.
Similar to CR-RLS, both algorithms modify the matrix $\msigmai{t}$ to
prevent its eigen-values to shrink to zero. The first algorithm,
described in \secref{sec:ARCOR}, extends AROWR to the non-stationary
setting and is similar in spirit to CR-RLS, yet the restart operations
it performs depend on the spectral properties of the covariance
matrix, rather than the time index $t$. Additionally, this algorithm
performs a projection of the weight vector into a predefined
ball. Similar technique was used in first order algorithms by Herbster
and Warmuth~\cite{HerbsterW01}, and Kivinen and Warmuth~\cite{Kiv_War}. Both steps are motivated both from the design and
analysis of AROWR. Its design is composed of solving small
optimization problems defined in~\eqref{arowr_objective}, one per
input example.  The non-stationary version performs explicit
corrections to its update, in order to prevent from the covariance
matrix to shrink to zero, and the weight-vector to grow too fast.

The second algorithm described in \secref{sec:LASER} is based on a
last-step min-max prediction principle and objective, where we replace
$L_T(\vu)$ in \eqref{minmax_algorithm_1} with $L_T(\{\vui{t}\})$ and
some additional modifications preventing the solution being
degenerate. Here the algorithmic modifications from the original AAR
algorithm are implicit and are due to the modifications of the
objective. The resulting algorithm smoothly interpolates the
covariance matrix with a unit matrix.

\subsection{ARCOR: Adaptive regularization of weights for Regression with COvariance Reset}
\label{sec:ARCOR}
Our first algorithm is based on the AROWR. We propose two
modifications to \eqref{update_rule_mean_rls} and
\eqref{update_rule_sigma}, which in combination overcome the problem
that the algorithm's learning rate gradually goes to zero.  The
modified algorithm operates on segments of the input sequence. In each
segment indexed by $i$, the algorithm checks whether the lowest
eigenvalue of $\msigmai{t}$ is greater than a given lower bound
$\Lambda_i$. Once the lowest eigenvalue of $\msigmaii$ is smaller than
$\Lambda_i$ the algorithm resets $\msigmaii=\mi$ and updates the value
of the lower bound $\Lambda_{i+1}$. Formally, the algorithm uses the
update \eqref{update_rule_sigma} to compute an intermediate candidate
for $\msigmai{t}$, denoted by
\begin{align}
\amsigmaii = \paren{\msigmai{t-1}^{-1} + \frac{1}{r} \vxi{t}\vxti{t}}^{-1} \label{sigma_alg}.
\end{align}
If indeed $\amsigmaii \succeq
\Lambda_i \mi$ then it sets $\msigmaii=\amsigmaii$, otherwise it sets
$\msigmaii=\mi$ and the segment index is increased by $1$.

Additionally, before our modification, the norm of the weight vector $\vwi{t}$ did
not increase much as the ``effective'' learning rate (the matrix
$\msigmai{t}$) went to zero. After our update, as the learning rate is
effectively bounded from below, the norm of $\vwi{t}$ may increase too
fast, which in turn will cause a low update-rate in non-stationarity
inputs.

We thus employ additional modification which is exploited by the
analysis. After updating the mean $\vwii$ as in \eqref{update_rule_mean_rls}
\begin{align}
\tvwi{t} & = \vwi{t-1}+\frac{(\yi{t}-\vxti{t}\vwi{t-1})\msigmai{t-1}\vxi{t}}{\cor+\vxti{t}\msigmai{t-1}\vxi{t}}\label{mu_alg},
\end{align}
we project it into a ball $B$ around the
origin of radius $\rb$ using a Mahalanobis distance. Formally, we
define the function $\proj(\tvw, \msigma, \rb)$ to be the solution of
the following optimization problem,
\begin{align*}
\arg\min_{ \Vert \vw\Vert\leq\rb} \half\paren{\vw-\tvw}^{\top}
\msigma^{-1} \paren{\vw-\tvw}
\end{align*}
We write the Lagrangian,
\begin{align*}
\mcal{L} = \half\paren{\vw-\tvw}^{\top}
\msigma^{-1} \paren{\vw-\tvw} + \alpha \paren{\half\Vert \vw\Vert^2-\half\rb^2 }~.
\end{align*}
Setting the gradient with respect to $\vw$ to zero we get,
\(
\msigma^{-1} \paren{\vw-\tvw} + \alpha \vw = 0
\).
Solving for $\vw$ we get
\begin{align*}
\vw = \paren{\alpha\mi + \msigma^{-1}}^{-1} \msigma^{-1}\tvw
= \paren{\mi +
  \alpha\msigma}^{-1}\tvw ~.
\end{align*}
From KKT conditions we get that if $\Vert\tvw\Vert\leq\rb$ then
$\alpha=0$ and $\vw=\tvw$. Otherwise, $\alpha$ is the unique positive
scalar that satisfies $\Vert \paren{\mi + \alpha\msigma}^{-1}\tvw \Vert = \rb$.
The value of $\alpha$ can be found using binary search and eigen-decomposition of the
matrix $\msigma$. We write explicitly $\msigma = \mv \mlam \mv^\top$
for a diagonal matrix $\mlam$. By
denoting $\vu = \mv^\top \tvw$  we rewrite the last equation, \( \Vert \paren{\mi + \alpha\mlam}^{-1}\vu\Vert = \rb \).
We thus wish to find $\alpha$ such that $\sum_j^d \frac{u_j^2}{(1+
  \alpha\mlam_{j,j})^2} = \rb^2$. It can be done using a binary
  search for $\alpha \in [0,a]$ where $a = (\Vert \vu \Vert/
  \rb-1)/\lambda_{\min}(\mlam)$.
To summarize, the projection step can be performed in time cubic in $d$
and logarithmic in $\rb$ and $\Lambda_i$.

We call the algorithm {ARCOR} for adaptive regularization with
covariance reset. A pseudo-code of the algorithm is summarized in the
left column of \tabref{table:algorithms}. We defer a comparison of
ARCOR and CR-RLS after the presentation of our second algorithm now.


\begin{center}
\begin{table*}[ht]
{\small
\hfill{}
\caption{ARCOR, LASER and CR-RLS algorithms} 
\label{table:algorithms} 
\begin{tabulary}{\textwidth}{|C|C|L|L|L|} 
\hline                       
 &  & \bf ARCOR & \bf LASER  & \bf CR-RLS \\ [0.5ex] 
\hline                  
 Parameters  & & $0<\cor,R_B$ ~,~ a sequence $1 > \Lambda_1 \geq
\Lambda_2 ... $ & $0<b<c$ & $0<\cor\leq1, T_0\in\mathbb{N}$ \\ 
\hline
 Initialize &  & $\vwi{0}=0 ~,~ \msigmai{0}=\mi ~,~ i=1$ &
 $\vwi{0}=0 ~,~ \msigmai{0}=\frac{c-b}{bc}\mi$ & $\vwi{0}=0 ~,~ \msigmai{0}=\mi$ \\ [0.5ex]
\hline
 & & \multicolumn{3}{|c|}{ Receive an instance $\vxi{t}$}  \\
\cline{2-5}
For & Output & \[\hyi{t}=\vxti{t}\vwi{t-1}\] & \[\hyi{t}=\frac{\vxti{t}\vwi{t-1}}{1+\vxti{t}\paren{\msigmai{t-1}+c^{-1}\mi}\vxi{t}}\] & \[\hyi{t}=\vxti{t}\vwi{t-1}\] \\
$t=1 ... T$ & prediction & & &\\
\cline{2-5}
 & & \multicolumn{3}{|c|}{Receive a correct label $\yi{t}$ }  \\
 \cline{2-5}
 & Update $\msigmai{t}$: & \[\tilde{\msigma}_{t}^{-1} = \msigmai{t-1}^{-1} + \frac{1}{\cor}\vxi{t}\vxti{t}\] & \[\msigma_{t}^{-1} = \paren{\msigmai{t-1}+c^{-1}\mi}^{-1}
+ \vxi{t}\vxti{t}\] & \[\tilde{\msigma}_{t}^{-1} = \cor\msigmai{t-1}^{-1} + \vxi{t}\vxti{t}\] \\
 & & \begin{tabular}{lll}
&If $\tilde{\msigma}_t \succeq \Lambda_i \mi$
   set $\msigmai{t}  = \tilde{\msigma}_t$\\
 &else set $\msigmai{t} =
  \mi~,~ i = i+1$
\end{tabular} & &
\begin{tabular}{lll}
If &$\mod(t,T_0)>0$\\
        &~~set $\msigmai{t}  = \tilde{\msigma}_t$\\
 else&~~set $\msigmai{t} =
  \mi$
\end{tabular}\\
\cline{2-5}
 & Update $\vwii$: & \vspace{0.2cm}\[\begin{array}{ll}\tvwi{t}  &\!\!\!\!=
   \vwi{t-1}\\&\!\!\!\!+\frac{(\yi{t}-\vxti{t}\vwi{t-1})\msigmai{t-1}\vxi{t}}{\cor+\vxti{t}\msigmai{t-1}\vxi{t}}\\~\\\vwii &\!\!\!\!= \proj\paren{{\tvwi{t}},\msigmaii, \rb}\end{array}\]
 & \[\!\!\!\!\!\!\begin{array}{ll}\vwi{t}  &\!\!\!\!\!= \vwi{t-1}\\&\!\!\!\!\!+\frac{(\yi{t}-\vxti{t}\vwi{t-1})\paren{\msigmai{t-1}+c^{-1}\mi}\vxi{t}}{1+\vxti{t}\paren{\msigmai{t-1}+c^{-1}\mi}\vxi{t}}\end{array}\]
& \[\begin{array}{ll}\vwi{t}  &\!\!\!\!= \vwi{t-1}\\&\!\!\!\!+\frac{(\yi{t}-\vxti{t}\vwi{t-1})\msigmai{t-1}\vxi{t}}{\cor+\vxti{t}\msigmai{t-1}\vxi{t}}\end{array}\] \\
 \hline
Output &  & $\vwi{T} \ ,\ \msigmai{T}$ & $\vwi{T} \ ,\ \msigmai{T}$ & $\vwi{T} \ ,\ \msigmai{T}$\\
 [1ex]     
\hline  
\end{tabulary}}
\hfill{}
\end{table*}
\end{center}

%% file: LASER.tex
\subsection{Last-Step Min-Max Algorithm for Non-stationary Setting}
\label{sec:LASER}
Our second algorithm is based on a last-step min-max predictor
proposed by Forster~\cite{Forster} and later modified by Moroshko and
Crammer~\cite{MoroshkoCr12} to obtain sub-logarithmic regret in the
stationary case.  On each round, the algorithm predicts as it is the
last round, and assumes a worst case choice of $\yi{t}$ given the
algorithm's prediction.

We extend this rule for the non-stationary setting given in
\eqref{minmax_algorithm_1}, and re-define the last-step minmax
predictor $\hyi{T}$ to be\footnote{$\yi{T}$ and $\hyi{T}$ serve both
  as quantifiers (over the $\min$ and $\max$ operators, respectively),
  and as the optimal arguments of this optimization problem. },
\begin{align}
\arg\min_{\hat{y}_{T}}\max_{y_{T}}\left[\sum_{t=1}^{T}\left(y_{t}-\hat{y}_{t}\right)^{2}\!-\!\!\!\min_{\vui{1},
    .. ,\vui{T}}
  \!\! Q_{T}\left(\vui{1}, ... ,\vui{T}\right)\right],\label{minmax_1}
\end{align}
where,
\begin{align}
Q_{t}\left(\vui{1},\ldots,\vui{t}\right) =& b\left\Vert
  \vui{1}\right\Vert ^{2}+c\sum_{s=1}^{t-1}\left\Vert
  \vui{s+1}-\vui{s}\right\Vert ^{2}
\nonumber\\
& +\sum_{s=1}^{t}\left(y_{s}-\vuti{s}\vxi{s}\right)^{2}~,\label{Q}
\end{align}
for some positive constants $b,c$. The first term of \eqref{minmax_1}
is the loss suffered by the algorithm while
$Q_{t}\left(\vui{1},\ldots,\vui{t}\right)$ defined in
\eqref{Q} is a sum of the loss suffered by some sequence of linear
functions $\left(\vui{1},\ldots,\vui{t}\right)$ and a
penalty for consecutive pairs that are far from each other, and for the
norm of the first to be far from zero.

We develop the algorithm by solving the three optimization problems in
\eqref{minmax_1}, first, minimizing the inner term,
$\min_{\vui{1}, .. ,\vui{T}}
Q_{T}\left(\vui{1}, ... ,\vui{T}\right)$, maximizing
over $\vyi{T}$, and finally, minimizing over $\hyi{T}$.  We start with
the inner term
for which we define an auxiliary function,
\begin{align*}
P_{t}\left(\vui{t}\right)=\min_{\vui{1},\ldots,\vui{t-1}}Q_{t}\left(\vui{1},\ldots,\vui{t}\right) ~,
\end{align*}
 which clearly satisfies,
\begin{equation*}
\min_{\vui{1},\ldots,\vui{t}}Q_{t}\left(\vui{1},\ldots,\vui{t}\right)
= \min_{\vui{t}} P_t(\vui{t}) ~.
\end{equation*}

The
following lemma states a recursive form of the function-sequence $P_t(\vui{t})$.
\begin{lemma}
\label{lem:lemma11}
For $t=2,3,\ldots$
%
\begin{align*}
P_1(\vui1)&=Q_1(\vui1)\\
 P_{t}\left(\vui{t}\right)&=
\min_{\vui{t-1}}\left(P_{t-1}\left(\vui{t-1}\right)
\!\!+\!\!c\left\Vert
    \vui{t}-\vui{t-1}\right\Vert
  ^{2} \!\!+\!\!\left(y_{t}\!-\!\vuti{t}\vxi{t}\right)^{2}\right).
\end{align*}
 \end{lemma}
The proof appears in \secref{proof_lemma11}. Using
\lemref{lem:lemma11} we write explicitly the function $P_t(\vui{t})$.
\begin{lemma}
\label{lem:lemma12}
The following equality holds
\begin{align}
P_{t}\left(\vui{t}\right)=\vuti{t}D_{t}\vui{t}-2\vuti{t}\vei{t}+f_{t}~,
\label{eqality_P}
\end{align}
where,
\begin{align}
&D_{1} \!=b\mi+\vxi{1}\vxti{1}
~,~~
D_{t}=\left(D_{t-1}^{-1}+c^{-1}\mi\right)^{-1}+\vxi{t}\vxti{t}\label{D}\\
&\vei{1}\!=y_{1}\vxi{1}~,~~
\vei{t}=\left(\mi+c^{-1}D_{t-1}\right)^{-1}\vei{t-1}+y_{t}\vxi{t}\label{e}\\
&f_{1}\!=y_{1}^{2}~,~~
f_{t}=f_{t-1}-\veti{t-1}\left(c\mi+D_{t-1}\right)^{-1}\vei{t-1}+y_{t}^{2}\label{f}
\end{align}
\end{lemma}
Note that $D_{t}\in\mathbb{R}^{d\times d}$ is a positive definite matrix,
$\vei{t}\in\mathbb{R}^{d\times1}$ and $f_{t}\in\mathbb{R}$. The proof appears in \secref{proof_lemma12}. From \lemref{lem:lemma12} we
conclude that,
\begin{align}
& \min_{\vui{1},\ldots,\vui{t}}Q_{t}\left(\vui{1},\ldots,\vui{t}\right)
=  \min_{\vui{t}}P_{t}\left(\vui{t}\right)\nonumber\\
& = \min_{\vui{t}}\left(\vuti{t}D_{t}\vui{t}-2\vuti{t}\vei{t}+f_{t}\right)
= -\veti{t}D_{t}^{-1}\vei{t}+f_{t}
 ~. \label{optimal_Q}
\end{align}
Substituting \eqref{optimal_Q} back in \eqref{minmax_1} we get that the last-step minmax
predictor is given by,
\begin{eqnarray}
\hyi{T}=\arg\min_{\hat{y}_{T}}\max_{y_{T}}\left[\sum_{t=1}^{T}\left(y_{t}-\hat{y}_{t}\right)^{2}+\veti{T}D_{T}^{-1}\vei{T}-f_{T}\right]
~. \label{minmax_2}
\end{eqnarray}
Since $\vei{T}$ depends on
$\yi{T} $ we substitute \eqref{e} in the second term
of \eqref{minmax_2}, 
\begin{align}
&\veti{T}D_{T}^{-1}\vei{T}=\nonumber\\
&\left(\left(\mi+c^{-1}D_{T-1}\right)^{-1}\vei{T-1}+y_{T}\vxi{T}\right)^{\top}D_{T}^{-1}\nonumber\\
&\left(\left(\mi+c^{-1}D_{T-1}\right)^{-1}\vei{T-1}+y_{T}\vxi{T}\right) ~.\label{second_term}
\end{align}
Substituting \eqref{second_term} and \eqref{f} in \eqref{minmax_2} and
omitting terms not depending explicitly on $y_{T}$ and $\hat{y}_{T}$
we get,
\begin{align}
\hat{y}_T
&= \arg\min_{\hat{y}_{T}}\max_{y_{T}}\bigg[\left(y_{T}-\hat{y}_{T}\right)^{2}  + y_{T}^{2}\vxti{T}D_{T}^{-1}\vxi{T}\nonumber\\
& \quad
+2y_{T}\vxti{T}D_{T}^{-1}\left(\mi+c^{-1}D_{T-1}\right)^{-1}\vei{T-1}-y_{T}^{2}\bigg]\nonumber\\
&=\arg\min_{\hat{y}_{T}}\max_{y_{T}}
\bigg[\left(\vxti{T}D_{T}^{-1}\vxi{T}\right)y_{T}^{2}\label{optimal_y}\\
& \quad
+2y_{T}\left(\vxti{T}D_{T}^{-1}\left(\mi+c^{-1}D_{T-1}\right)^{-1}\vei{T-1}-\hat{y}_{T}\right)+\hat{y}_{T}^{2}\bigg]~.\nonumber
\end{align}
The last equation is strictly convex in $\yi{T}$ and thus the optimal
solution is not bounded. To solve it, we follow an approach used by Forster in a different context~\cite{Forster}. In order to make the optimal
value bounded, we assume that the adversary can only
choose labels from a bounded set $\yi{T}\in[-Y,Y]$. Thus, the optimal
solution of \eqref{optimal_y} over $\yi{T}$ is given by the following
equation, since the optimal value is $\yi{T}\in\{+Y,-Y\}$,
\begin{align*}
\hat{y}_T
&=\arg\min_{\hat{y}_{T}}\bigg[\left(\vxti{T}D_{T}^{-1}\vxi{T}\right)
  Y^2 \\
& \quad +2 Y\left\vert
  \vxti{T}D_{T}^{-1}\left(\mi+c^{-1}D_{T-1}\right)^{-1}\vei{T-1}-\hat{y}_{T}\right\vert+\hat{y}_{T}^{2}\bigg]~.
\end{align*}
This problem is of a similar form to the one discussed by Forster~\cite{Forster}, from which we get the optimal solution,
\(
\hyi{T} = clip\paren{
  \vxti{T}D_{T}^{-1}\left(\mi+c^{-1}D_{T-1}\right)^{-1}\vei{T-1},Y}.
\)

The optimal solution depends explicitly on the bound $Y$, and as its
value is not known, we thus ignore it, and define the output of
the algorithm to be, 
\begin{align}
\hyi{T} &=
\vxti{T}D_{T}^{-1}\left(\mi+c^{-1}D_{T-1}\right)^{-1}\vei{T-1}\nonumber\\&=\vxti{T}D_{T}^{-1}D'_{T-1}\vei{T-1}
~, \label{my_predictor}
\end{align}
where we define
\begin{equation}
D'_{t-1}=\left(\mi+c^{-1}D_{t-1}\right)^{-1}~.
\label{D_prime}
\end{equation}
We call the algorithm {LASER} for last step adaptive regressor
algorithm. Clearly, for $c=\infty$ the LASER algorithm reduces to the AAR algorithm. Similar to CR-RLS and ARCOR, this algorithm can be also
expressed in terms of weight-vector $\vwi{t}$ and a PSD matrix
$\msigmaii$, by denoting $\vwi{t}=D_{t}^{-1}\vei{t}$
and $\msigmaii=D_{t}^{-1}$. The algorithm is summarized in
the middle column of \tabref{table:algorithms}.

%

%% file: discussion.tex
\subsection{Discussion}
\label{sec:discussion}
\tabref{table:algorithms} enables us to compare the three algorithms head-to-head. All algorithms perform linear predictions, and then update the prediction vector $\vwi{t}$ and the matrix $\msigmai{t}$.
CR-RLS and ARCOR are more similar to each other, both stem from a stationary algorithm, and perform resets from time-to-time. For CR-RLS it is performed every fixed time steps, while for ARCOR it is performed when the eigenvalues of the matrix (or effective learning rate) are too small. ARCOR also performs a projection step, which is motivated to ensure that the weight-vector will not grow to much, and is used explicitly in the analysis below.
Note that CR-RLS (as well as RLS) also uses a forgetting factor (if $\cor<1$).

Our second algorithm, LASER, controls the covariance matrix in a smoother way. On each iteration it interpolates it with the identity matrix before adding $\vxi{t}\vxti{t}$. Note that if $\lambda$ is an eigenvalue of $\msigmai{t-1}^{-1}$ then $\lambda\times\paren{c/(\lambda+c)} < \lambda$ is an eigenvalue of $\paren{\msigmai{t-1} + c^{-1} \mi}^{-1}$. Thus the algorithm implicitly reduce the eigenvalues of the inverse covariance (and increase the eigenvalues of the covariance).



Finally, all three algorithms can be combined with Mercer kernels as they employ only
sums of inner- and outer-products of its inputs. This allows them to perform non-linear predictions, similar to SVM. 


%% file: related.tex
\section{Related Work}
\label{sec:related_work}
There is a large body of research in online learning for regression
problems. Almost half a century ago, Widrow and Hoff~\cite{WidrowHoff}
developed a variant of the least mean squares (LMS) algorithm for
adaptive filtering and noise reduction. The algorithm was further
developed and analyzed extensively (for example by Feuer~\cite{Feuer85}). The
normalized least mean squares filter (NLMS)~\cite{Bershad,Bitmead}
builds on LMS and performs better to scaling of the input. The
recursive least squares (RLS)~\cite{Hayes} is the closest to our
algorithms in the signal processing literature and also maintains a
weight-vector and a covariance-like matrix, which is positive
semi-definite (PSD), that is used to re-weight inputs.

In the machine learning literature the problem of online regression
was studied extensively, and clearly we cannot cover all the relevant
work. Cesa-Bianchi et al.~\cite{Nicolo_Warmuth} studied gradient
descent based algorithms for regression with the squared loss.
Kivinen and Warmuth~\cite{Kiv_War} proposed various generalizations for
general regularization functions. We refer the reader to a
comprehensive book in the subject~\cite{CesaBiGa06}.

Foster~\cite{Foster91} studied an online version of the ridge
regression algorithm in the worst-case setting.  Vovk~\cite{vovkAS}
proposed a related algorithm called the Aggregating Algorithm (AA), which was later applied to the problem of linear regression with square loss~\cite{Vovk01}. 
Forster~\cite{Forster} simplified the regret analysis for this problem. Both algorithms employ second order information. ARCOR
for the separable case is very similar to these algorithms, although
has alternative derivation. Recently, few algorithms were proposed
either for
classification~\cite{CesaBianchiCoGe05,DredzeCrPe08,CrammerDrPe08,CrammerKuDr09}
or for general loss functions~\cite{DuchiHS10,McMahanS10} in the
online convex programming framework.  AROWR~\cite{VaitsCr11} shares
the same design principles of AROW~\cite{CrammerKuDr09} yet it is
aimed for regression. The ARCOR algorithm takes AROWR one step further
and it has two important modifications which makes it work in the
drifting or shifting settings. These modifications make the analysis
more complex than of AROW.

Two of the approaches used in previous algorithms for non-stationary
setting are to bound the weight vector and covariance reset. Bounding
the weight vector was performed either by projecting it into a bounded
set~\cite{HerbsterW01}, shrinking it by
multiplication~\cite{KivinenSW01}, or subtraction of previously seen
examples~\cite{CavallantiCG07}. These three methods (or at least most
of their variants) can be combined with kernel operators, and in fact,
the last two approaches were designed and motivated by kernels.

The Covariance Reset RLS algorithm
(CR-RLS)~\cite{Salgado,Goodwin,Chen} was designed for adaptive
filtering. CR-RLS makes covariance reset every fixed amount of data
points, while ARCOR performs restarts based on the actual properties
of the data - the eigenspectrum of the covariance matrix. Furthermore,
as far as we know, there is no analysis in the mistake bound model for
this algorithm.  Both ARCOR and CR-RLS are motivated from the property
that the covariance matrix goes to zero and becomes rank deficient.
In both algorithms the information encapsulated in the covariance
matrix is lost after restarts. In a rapidly varying environments, like
a wireless channel, this loss of memory can be beneficial, as previous
contributions to the covariance matrix may have little correlation
with the current structure. Recent versions of
CR-RLS~\cite{Goodhart,Song1042669} employ covariance reset to have
numerically stable computations.

ARCOR algorithm combines both techniques with online learning that
employs second order algorithm for regression. In this aspect we have
the best of all worlds, fast convergence rate due to the usage of
second order information, and the ability to adapt in non-stationary
environments due to projection and resets.

LASER is simpler than all these algorithms as it controls the increase
of the eigenvalues of the covariance matrix implicitly rather than
explicitly by ``averaging'' it with a fixed diagonal matrix (see
\eqref{D}), and it do not involve projection steps.  The Kalman
filter~\cite{Kalman60} and the $H_\infty$ algorithm~(e.g. the work of
Simon~\cite{Simon:2006:OSE:1146304}) designed for filtering take a
similar approach, yet the exact algebraic form is different.

The derivation of the LASER algorithm in this work shares similarities
with the work of Forster~\cite{Forster} and the work of Moroshko and Crammer~\cite{MoroshkoCr12}.  These algorithms are
motivated from the last-step min-max predictor. Yet, the algorithms of
Forster and Moroshko and Crammer are designed for the stationary setting, while LASER is
primarily designed for the non-stationary setting. Moroshko and
Crammer~\cite{MoroshkoCr12} also discussed a weak variant of the
non-stationary setting, where the complexity is measured by the total
distance from a reference vector $\bar{\mathbf{u}}$, rather than the
total distance of consecutive vectors (as in this paper), which is
more relevant to non-stationary problems.


%% file: analysis_ARCOR.tex
\section{Regret bounds}
\label{sec:regret}
We now analyze our algorithms in the non-stationary case, upper
bounding the regret using more than a single comparison
vector. Specifically, our goal is to prove bounds that would hold
uniformly for all inputs, and are of the form,
\[
 L_T(\textrm{alg})  \leq L_T( \{ \vui{t} \} ) + \alpha(T) \paren{V^{(P)}}^\gamma + \beta(T) ~,
\]
for either $P=1$ or $P=2$, a constant $\gamma$ and some functions
$\alpha(T),\beta(T)$ that may depend implicitly on other quantities of
the problem.

Specifically, in the next section we show that under a particular
choice of $\Lambda_i=\Lambda_i(V^{(1)})$ for the ARCOR algorithm, its
regret is bounded by,
\[
 L_T(\textrm{ARCOR})  \leq L_T( \{ \vui{t} \} )+
 \mcal{O} \paren{ T^{\half}\paren{ V^{(1)} }^{\half}\log T}~.
\]
Additionally, in \secref{sec:analysis_laser}, we show that under proper
choice of the constant $c=c\paren{V^{(2)}}$, the regret of LASER is bounded by,
\[
 L_T(\textrm{LASER})  \leq L_T( \{ \vui{t} \} )+
 \mcal{O} \paren{ T^\frac{2}{3}\paren{ V^{(2)}}^\frac{1}{3}}~.
\]
The two bounds are not comparable in general. For example, assume a
constant instantaneous drift $\Vert\vui{t+1}-\vui{t}\Vert=\nu$ for
some constant value $\nu$. In this case the variance and squared variance are,
$V^{(1)}=T\nu$ and $V^{(2)}=T\nu^2$. The bound of ARCOR becomes
$ \nu^{\half}T\log T$, while the bound of LASER becomes $\nu^\frac{2}{3}T$.
The bound of ARCOR is larger if $(\log T)^6 > \nu$, and the bound of
LASER is larger in the opposite case.

Another example is polynomial decay of the drift,
$\Vert\vui{t+1}-\vui{t}\Vert \leq t^{-\kappa}$ for some $\kappa > 0$.
In this case, for $\kappa \neq 1$ we get $V^{(1)} \leq \sum^{T-1}_{t=1} t^{-\kappa} \leq \int^{T-1}_1 t^{-\kappa}  dt+1=\frac{(T-1)^{1-\kappa}-\kappa}{1-\kappa}$. For $\kappa=1$ we get $V^{(1)} \leq \log(T-1)+1$.
For LASER we have, for $\kappa \neq 0.5$, $V^{(2)} \leq \sum^{T-1}_{t=1} t^{-2\kappa} \leq \int^{T-1}_1 t^{-2\kappa}  dt+1=\frac{(T-1)^{1-2\kappa}-2\kappa}{1-2\kappa}$. For $\kappa=0.5$ we get $V^{(2)} \leq \log(T-1)+1$.
Asymptotically, ARCOR outperforms LASER about when $\kappa \geq 0.7$.

Herbster and
Warmuth~\cite{HerbsterW01} developed shifting bounds for general
gradient descent algorithms with projection of the weight-vector using
the Bregman divergence. In their bounds, there is a factor greater
than 1 multiplying the term $L_T\paren{\braces{\vui{t}}}$, leading to a
small regret only when the data is close to be realizable with
linear models. Busuttil and Kalnishkan~\cite{BusuttilK07} developed a variant of the Aggregating Algorithm~\cite{vovkAS} for the non-stationary
setting. However, to have sublinear regret they require a strong
assumption on the drift $V^{(2)}=o(1)$, while we require only
$V^{(2)}=o(T)$ (for LASER) or $V^{(1)}=o(T)$ (for ARCOR).


\subsection{Analysis of ARCOR algorithm}
\label{sec:analysis_arcor}
Let us define additional notation that we will use in our bounds. We denote by $t_i$ the example index for which a restart
was performed for the i$th$ time, that is $\msigmai{t_i}=\mi$ for all
$i$. We define by $n$ the total number of restarts, or intervals. We
denote by $T_i = t_i - t_{i-1}$ the number of examples between two
consecutive restarts. Clearly $T = \sum_{i=1}^n T_i$. Finally, we
denote by $\msigma^{i-1} = \msigmai{t_i-1}$ just before the i$th$
restart, and we note that it depends on exactly $T_i$ examples (since
the last restart).

In what follows we compare the performance of the ARCOR algorithm to
the performance of a sequence of weight vectors $\vui{t} \in \reals^d$
all of which are of bounded norm $\rb$. In other words, all the
vectors $\vui{t}$ belong to $B$. We break the proof into four steps.
In the first step (\thmref{non_stat_bound}) we bound the regret when
the algorithm is executed with some value of parameters $\{ \Lambda_i
\}$ and the resulting covariance matrices. In the second step,
summarized in \corref{cor:remove_sigma}, we remove the dependencies in
the covariance matrices, by taking a worst case bound. In the third
step, summarized in \lemref{lem:number_of_segments}, we upper bound the
total number of switches $n$ given the parameters $\{ \Lambda_i
\}$. Finally, in \corref{cor:final_cor} we provide the regret bound
for a specific choice of the parameters. We now move to state the
first theorem.
\begin{theorem}\label{non_stat_bound}
  Assume that the ARCOR algorithm is run with an
  input sequence $(\vxi{1},\yi{1}) \comdots (\vxi{T}, \yi{T})$. Assume
  that all the inputs are upper bounded by unit norm $\Vert\vxi{t}\Vert\leq1$ and that
  the outputs are bounded by $Y=\max_t \vert\yi{t}\vert$. Let $\vui{t}$
  be any sequence of bounded weight vectors $\Vert\vui{t}\Vert \leq
  \rb$. Then, the cumulative loss is bounded by,
\begin{align*}
\begin{split}
L_T(&\textrm{ARCOR})  \leq L_T(\{ \vui{t} \}) + 2\rb\cor\sum_t
  \frac{1}{\Lambda_{i(t)}}\Vert\vui{t-1}-\vui{t}\Vert\\
& + \cor  \vuti{T}\msigmai{T}^{-1}\vui{T}
+ 2\paren{\rb^2+Y^2}\sum_i^n
\log\det\paren{\paren{\msigma^i}^{-1}}
\end{split}
\end{align*}
where $n$ is the number of covariance restarts and $\msigma^{i-1}$ is
the value of the covariance matrix just before the $i$th restart.
\end{theorem}
The proof appears in \secref{sec:ARCOR_Theoreme_proof}.
Note that the number of restarts $n$ is not fixed but depends both on
the total number of examples $T$ and the scheme used to set the values
of the lower bound of the eigenvalues $\Lambda_i$. In general, the
lower the values of $\Lambda_i$ are, the smaller number of
covariance-restarts occur, yet the larger the value of the last term
of the bound is, which scales inversely proportional to $\Lambda_i$. A
more precise statement is given in the next corollary.
\begin{corollary}\label{cor:remove_sigma}
Assume that the ARCOR algorithm made $n$ restarts. Under the conditions of \thmref{non_stat_bound} we have,
\begin{align*}
\begin{split}
L_T(&\textrm{ARCOR})  \leq L_T(\{ \vui{t} \})
+ 2\rb\cor \Lambda_{n}^{-1}\sum_t
  \Vert\vui{t-1}-\vui{t}\Vert \\
&+2\paren{\rb^2+Y^2}  d n \log\paren{1 +  \frac{T}{n \cor d}}
+ \cor  \vuti{T}\msigmai{T}^{-1}\vui{T}
\end{split}
\end{align*}
\label{cor:bound_T}
\end{corollary}
\begin{proof}
By definition we have
\[
\paren{\msigma^i}^{-1} = \mi +
  \frac{1}{\cor}\sum_{t=t_i}^{T_i+t_i} \vxi{t}\vxti{t}~.
\]
Denote the
  eigenvalues of $\sum_{t=t_i}^{T_i+t_i} \vxi{t}\vxti{t}$ by
  $\lambda_1 \comdots \lambda_d$. Since $\Vert\vxi{t}\Vert\leq1$ their
  sum is $\tr\paren{\sum_{t=t_i}^{T_i+t_i} \vxi{t}\vxti{t}} \leq T_i$. We
  use the concavity of the $\log$ function to bound \(
  \log\det\paren{\paren{\msigma^i}^{-1}} = \sum_j^d \log\paren{1 +
    \frac{\lambda_j}{\cor}}\leq d \log\paren{1 + \frac{T_i}{\cor d}}.\) We use concavity again to bound the sum
\begin{align*}
\sum_i^n
\log\det\paren{\paren{\msigma^i}^{-1}} & \leq \sum_i^n d \log\paren{1 +
  \frac{T_i}{\cor d}}\\
  & \leq d n \log\paren{1 +  \frac{T}{n \cor d}}  ~,
\end{align*}
where we used the fact that $\sum_i^n T_i = T$. Substituting the last
inequality in \thmref{non_stat_bound}, as well as using the
monotonicity of the coefficients, $\Lambda_i \geq \Lambda_n$ for all
$i\leq n$, yields the desired bound.
\end{proof}
Implicitly, the second and third terms of the bound have opposite
dependence on $n$. The second term is decreasing with $n$. If $n$ is
small it means that the lower bound $\Lambda_n$ is very low (otherwise
we would make many restarts) and thus $\Lambda_n^{-1}$ is large.  The
third term is increasing with $n\ll T$.  We now make this implicit
dependence explicit.

Our goal is to bound the number of restarts $n$ as a function of the
number of examples $T$. This depends on the exact sequence of values
$\Lambda_i$ used. The following lemma provides a bound  on $n$ given
a specific sequence of $\Lambda_i$.
\begin{lemma}
\label{lem:number_of_segments}
Assume that the ARCOR algorithm is run with some
sequence of $\Lambda_i$. Then, the number of restarts is
upper bounded by,
\[
n \leq \max_N \braces{ N~:~ T \geq \cor\sum_i^N \paren{\Lambda_i^{-1}
    - 1}} ~.
\]
\label{lem:bound_intervals}
\end{lemma}
\begin{proof}
  Since $\sum_{i=1}^n T_i = T$, then the number of restarts is
  maximized when the number of examples between restarts $T_i$ is
  minimized. We prove now a lower bound on $T_i$ for all $i=1 \dots
  n$.  A restart occurs for the i$th$ time when the smallest
  eigenvalue of $\msigmai{t}$ is smaller (for the first time) than
  $\Lambda_i$.

  As before, by definition, $\paren{\msigma^i}^{-1} = \mi +
  \frac{1}{\cor}\sum_{t=t_i}^{T_i+t_i} \vxi{t}\vxti{t}$. By
  a result in matrix analysis~\cite[Theorem~8.1.8]{Golub:1996:MC:248979} we have that there
  exists a matrix $A\in\reals^{d \times T_i}$ with each column belongs
  to a bounded convex body that satisfy $a_{k,l}\geq 0$ and
  $\sum_k a_{k,l}\leq 1$ 
  for $l=1 \comdots T_i$, such that the $k$th
  eigenvalue $\lambda_k^i$ of $\paren{\msigma^i}^{-1} $ equals to
  $\lambda_k^i = 1 + \frac{1}{\cor}\sum_{l=1}^{T_i} a_{k,l}$.
  The value of $T_i$ is defined as when largest eigenvalue of
  $\paren{\msigma^i}^{-1} $ hits $\Lambda_i^{-1}$.  Formally, we get
  the following lower bound on $T_i$,
\begin{align*}
\arg\min_{\{a_{k,l}\}} ~ s &\textrm{~~~s.t.} &~ \max_k \paren{1 + \frac{1}{\cor}\sum_{l=1}^{s}
  a_{k,l}} \geq \Lambda_i^{-1}\\
&&~ a_{k,l} \geq 0 \quad \textrm{ for } k=1\comdots d,l=1\comdots s\\
&&~ \sum_k a_{k,l} \leq 1\quad \textrm{ for } l\leq1\comdots s
\end{align*}
For a fixed value of $s$, a maximal value $\max_k \paren{1 + \frac{1}{\cor}\sum_{l=1}^{s}
  a_{k,l}} $ is obtained if all the ``mass'' is concentrated in one
value $k$ and for this $k$ each $a_{k,l}$ is equal to its maximal value $1$. That is, we have $a_{k,l} = 1$ for $k=k_0$ and $a_{k,l}=0$
otherwise.
In this case $\max_k \paren{1 + \frac{1}{\cor}\sum_{l=1}^{s}
  a_{k,l}}  = \paren{1 + \frac{1}{\cor} s}$ and the lower bound is
  obtained when $\paren{1 + \frac{1}{\cor} s} = \Lambda_i^{-1}$.
Solving for $s$ we get that the shortest possible length of the i$th$
interval is bounded by,
\(
T_i \geq  \cor\paren{\Lambda_i^{-1} - 1}.
\)
Summing over the last equation we get,
\(
T = \sum_i^n T_i \geq \cor\sum_i^n \paren{\Lambda_i^{-1} - 1}.
\)
Thus, the number of restarts is upper bounded by the maximal value
$n$ that satisfy the last inequality.
\end{proof}

We now prove a bound for a specific choice of the parameters
$\{\Lambda_i\}$, namely polynomial decay, $\Lambda_i^{-1} = i^{q-1} +
1$. This schema to set $\{\Lambda_i\}$ balances between the amount of
noise (need for many restarts) and the property that using the
covariance matrix for updates achieves fast-convergence. We note that
an exponential schema $\Lambda_i = 2^{-i}$ will lead to very few
restarts, and very small eigenvalues of the covariance
matrix. Intuitively, this is because the last segment will be about
half the length of the entire sequence.
Combining \lemref{lem:bound_intervals} with \corref{cor:bound_T} we get,
\begin{corollary}
\label{cor:final_cor}
Assume that the ARCOR algorithm is run with a
polynomial schema, that is $\Lambda_i^{-1} = i^{q-1} + 1$ for some $q
\neq 0$. Under the conditions of \thmref{non_stat_bound} we have,
\begin{align}
L_T&(\textrm{ARCOR})  \leq L_T(\{ \vui{t} \})
+ \cor
\vuti{T}\msigmai{T}^{-1}\vui{T}\nonumber\\
&+
 2\paren{\rb^2+Y^2}  d  \paren{q T+1}^{\frac{1}{q}} \log\paren{1 +  \frac{T}{n \cor d}}\label{poly_bound_1}\\
&
+ 2\rb\cor \paren{\paren{q T\!+\!1}^{\frac{q-1}{q}}\!+\!1}\sum_t
  \Vert\vui{t-1}-\vui{t}\Vert.\label{poly_bound_2}
\end{align}
\label{cor:ploy_bound}
\end{corollary}
\begin{proof}
Substituting $\Lambda_i^{-1} = i^{q-1} + 1$ in \lemref{lem:bound_intervals} we
get,
\begin{align*}
T \geq \cor\sum_i^n \paren{\Lambda_i^{-1} - 1} = \cor \sum_{i=1}^n i^{q-1}
\geq \cor\int_1^{n} x^{q-1} dx = \frac{\cor}{q} \paren{n^{q}-1}~.
\end{align*}
this yields an upper bound on $n$,
\begin{align*}
n \leq \paren{q T+1}^\frac{1}{q} \quad\Rightarrow\quad
\Lambda_n^{-1}  \leq \paren{qT+1}^{\frac{q-1}{q}}+1
\end{align*}
\end{proof}
Comparing the last two terms of the bound of \corref{cor:ploy_bound}  we observe a
natural tradeoff in the value of $q$. The third term of
\eqref{poly_bound_1} is decreasing with large values of $q$, while the
fourth term of \eqref{poly_bound_2} is increasing with $q$.

Assuming a bound on the deviation $\sum_t \Vert\vui{t-1}-\vui{t}\Vert
= V^{(1)}_T \leq \mcal{O}\paren{T^{1/p}}$, or in other words $p =
\paren{\log T}/\paren{\log V^{(1)}}$. We set a drift dependent parameter $q =
\paren{2p}/\paren{p+1} = \paren{2 \log T}/\paren{\log T + \log V^{(1)}}$ and
get that the sum of \eqref{poly_bound_1} and \eqref{poly_bound_2} is
of order $\mcal{O}\paren{T^{\frac{p+1}{2p}} \log (T)} = \mcal{O} \paren{
  \sqrt{ V^{(1)} T}\log T}$.

Few comments are in order. First, as long as $p>1$ the sum of
\eqref{poly_bound_1} and \eqref{poly_bound_2} is ${o}(T)$ and thus
vanishing.
Second, when the noise is very low, that is $p\approx -(1+\eps)$, the
algorithm sets $q\approx 2+(2/\eps)$, and thus it will not make any
restarts, and the bound of $\mcal{O}(\log T)$ for the
stationary case is retrieved. In other words, for this choice of $q$ the
algorithm will have only one interval, and there will be no restarts.

To conclude, we showed that if the algorithm is given an upper bound
on the amount of drift, which is sub-linear in $T$, it can achieve
sub-linear regret.  Furthermore, if it is known that there is no
non-stationarity in the reference vectors, then running the algorithm
with large enough $q$ will have a regret logarithmic in $T$.

%% file: analysis_LASER.tex
\subsection{Analysis of LASER algorithm}
\label{sec:analysis_laser}
We now analyze the performance of the LASER algorithm in the worst-case
setting in six steps. First, state a technical lemma that is used in
the second step (\thmref{thm:basic_bound}), in which we bound the regret with a quantity
proportional to
$\sum_{t=1}^{T}\vxti{t}D_{t}^{-1}\vxi{t}$. Third,
in \lemref{lem:bound_1} we bound each of the summands with two terms,
one logarithmic and one linear in the eigenvalues of the matrices
$D_{t}$. In the fourth (\lemref{operator_scalar}) and fifth (\lemref{eigen_values_lemma})
steps we bound the eigenvalues of $D_{t}$ first for scalars
and then extend the results to matrices. Finally, in \corref{cor:main_corollary} we put
all these results together and get the desired bounds.

\begin{lemma}
For all $t$ the following statement holds,
\begin{align*}
D'_{t-1}D_{t}^{-1}\vxi{t}\vxti{t}D_{t}^{-1}D'_{t-1}
+D'_{t-1}\left(D_{t}^{-1}D'_{t-1}+c^{-1}\mi\right)\\-D_{t-1}^{-1}
\preceq0
\end{align*}
where as defined in \eqref{D_prime} we have
\(
D'_{t-1}=\left(\mi+c^{-1}D_{t-1}\right)^{-1}~.
\)
\label{lem:technical}
\end{lemma}
The proof appears in \secref{proof_lemma_technical}.
We next bound the cumulative loss of the algorithm,
\begin{theorem}
\label{thm:basic_bound}
Assume that the labels are bounded $\sup_t \vert \yi{t} \vert \leq Y$ for some
$Y\in\reals$. Then the following bound holds,
\begin{align}
L_T(\textrm{LASER})&\leq\min_{\vui{1},\ldots,\vui{T}}\left[
L_T(\{\vui{t}\})
+c V^{(2)}_T(\{\vui{t}\})
+b\left\Vert
    \vui{1}\right\Vert ^{2}
\right]\nonumber\\
& \quad +Y^{2}\sum_{t=1}^{T}\vxti{t}D_{t}^{-1}\vxi{t} ~.\label{bound_1}
\end{align}
\end{theorem}
\begin{proof}
Fix $t$. A long algebraic manipulation, given in
\secref{algebraic_manipulation} yields,
\begin{align}
&\left(y_{t}-\hat{y}_{t}\right)^{2}+\min_{\vui{1},\ldots,\vui{t-1}}Q_{t-1}\left(\vui{1},\ldots,\vui{t-1}\right)\nonumber\\
&~~~~~~~~~~~~ -\min_{\vui{1},\ldots,\vui{t}}Q_{t}\left(\vui{1},\ldots,\vui{t}\right)\nonumber\\
=&\left(y_{t}-\hat{y}_{t}\right)^{2}
+2y_{t}\vxti{t}D_{t}^{-1}D'_{t-1}\vei{t-1}\nonumber\\
&\!+\!\veti{t-1}\!\Bigg[ \!-\!D_{t-1}^{-1} 
\!+\!\!
D'_{t-1}
\!\left(D_{t}^{-1}
D'_{t-1}
\!\!
\!+\!c^{-1}\mi\right)\!\!\Bigg]\!\vei{t-1}\nonumber\\
&+y_{t}^{2}\vxti{t}D_{t}^{-1}\vxi{t}-y_{t}^{2}~.
\label{step1}
\end{align}
Substituting the specific value of the predictor
$\hat{y}_{t}=\vxti{t}D_{t}^{-1}D'_{t-1}\vei{t-1}$
from \eqref{my_predictor}, we get that \eqref{step1} equals to,
\begin{align}
  &
  \hat{y}_{t}^{2}+y_{t}^{2}\vxti{t}D_{t}^{-1}\vxi{t}\nonumber\\
&+\veti{t-1}\Bigg[-D_{t-1}^{-1}+D'_{t-1}\left(D_{t}^{-1}D'_{t-1}+c^{-1}\mi\right)\Bigg]\vei{t-1}\nonumber\\
 =~&
 \veti{t-1}D'_{t-1}D_{t}^{-1}\vxi{t}\vxti{t}D_{t}^{-1}D'_{t-1}\vei{t-1}+y_{t}^{2}\vxti{t}D_{t}^{-1}\vxi{t}\nonumber\\
&+\veti{t-1}\Bigg[-D_{t-1}^{-1}+D'_{t-1}\left(D_{t}^{-1}D'_{t-1}+c^{-1}\mi\right)\Bigg]\vei{t-1}
 \nonumber\\
 =~& \veti{t-1} \tilde{D}_t \vei{t-1}+y_{t}^{2}\vxti{t}D_{t}^{-1}\vxi{t} ~,\label{step2}
\end{align}
 where $\tilde{D}_t =
 D'_{t-1}D_{t}^{-1}\vxi{t}\vxti{t}D_{t}^{-1}D'_{t-1}-D_{t-1}^{-1}+D'_{t-1}\left(D_{t}^{-1}D'_{t-1}+c^{-1}\mi\right)$.
Using \lemref{lem:technical} we upper bound $\tilde{D}_t
\preceq 0$ and thus \eqref{step2} is bounded,
\[
y_{t}^{2}\vxti{t}D_{t}^{-1}\vxi{t} \leq
Y^{2}\vxti{t}D_{t}^{-1}\vxi{t} ~.
\]
Finally, summing over $t\in\left\{ 1,\ldots,T\right\} $ gives the
desired bound,
\begin{align*}
&L_T(\textrm{LASER})-\min_{\vui{1},\ldots,\vui{T}}\left[b\left\Vert
    \vui{1}\right\Vert ^{2}
+c V^{(2)}_T(\{\vui{t}\})
+ L_T(\{\vui{t}\})
\right]\\
\leq&
Y^{2}\sum_{t=1}^{T}\vxti{t}D_{t}^{-1}\vxi{t} ~.
\end{align*}
\end{proof}

In the next lemma we further bound the right term of
\eqref{bound_1} 
This type of bound is based on the
usage of the covariance-like matrix $D$.
\begin{lemma}
\label{lem:bound_1}
\begin{equation}
\sum_{t=1}^{T}\vxti{t}D_{t}^{-1}\vxi{t}\leq\ln\left|\frac{1}{b}D_{T}\right|
+c^{-1}\sum_{t=1}^{T} \tr\paren{D_{t-1}} ~.\label{covariance_bound}
\end{equation}
\end{lemma}
\begin{proof}
Let $B_{t}\doteq D_{t}-\vxi{t}\vxti{t}=\left(D_{t-1}^{-1}+c^{-1}\mi\right)^{-1}\succ0$.
\begin{align*}
\vxti{t}D_{t}^{-1}\vxi{t} & =  \tr\left(\vxti{t}D_{t}^{-1}\vxi{t}\right)=\tr\left(D_{t}^{-1}\vxi{t}\vxti{t}\right)\\
 & = \tr\left(D_{t}^{-1}\left(D_{t}-B_{t}\right)\right)\\
 & = \tr\left(D_{t}^{-1/2}\left(D_{t}-B_{t}\right)D_{t}^{-1/2}\right)\\
 & = \tr\left(\mi-D_{t}^{-1/2}B_{t}D_{t}^{-1/2}\right)\\
 & =
 \sum_{j=1}^{d}\left[1-\lambda_{j}\left(D_{t}^{-1/2}B_{t}D_{t}^{-1/2}\right)\right] ~.
\end{align*}
We continue using $1-x\leq-\ln\left(x\right)$ and get
\begin{align*}
\vxti{t}D_{t}^{-1}\vxi{t} & \leq  -\sum_{j=1}^{d}\ln\left[\lambda_{j}\left(D_{t}^{-1/2}B_{t}D_{t}^{-1/2}\right)\right]\\
 & = -\ln\left[\prod_{j=1}^{d}\lambda_{j}\left(D_{t}^{-1/2}B_{t}D_{t}^{-1/2}\right)\right]\\
 & = -\ln\left|D_{t}^{-1/2}B_{t}D_{t}^{-1/2}\right|\\
 & =
 \ln\frac{\left|D_{t}\right|}{\left|B_{t}\right|}=\ln\frac{\left|D_{t}\right|}{\left|D_{t}-\vxi{t}\vxti{t}\right|} ~.
\end{align*}
It follows that,
\begin{align*}
\vxti{t}D_{t}^{-1}\vxi{t}
& \leq
\ln\frac{\left|D_{t}\right|}{\left|\left(D_{t-1}^{-1}+c^{-1}\mi\right)^{-1}\right|}\\
 & =  \ln\frac{\left|D_{t}\right|}{\left|D_{t-1}\right|}\left|\left(\mi+c^{-1}D_{t-1}\right)\right|\\
 & =
\ln\frac{\left|D_{t}\right|}{\left|D_{t-1}\right|}+\ln\left|\left(\mi+c^{-1}D_{t-1}\right)\right| ~.
\end{align*}
 and because $\ln\left|\frac{1}{b}D_{0}\right| \geq 0$ we get
\begin{align*}
\sum_{t=1}^{T}\vxti{t}D_{t}^{-1}\vxi{t}&\leq
\ln\left|\frac{1}{b}D_{T}\right|
+\sum_{t=1}^{T}\ln\left|\left(\mi+c^{-1}D_{t-1}\right)\right|\\
& \leq \ln\left|\frac{1}{b}D_{T}\right|
 +c^{-1}\sum_{t=1}^{T} \tr\paren{D_{t-1}} ~.
\end{align*}
\end{proof}

At first sight it seems that the right term of
\eqref{covariance_bound} may grow super-linearly with $T$, as each of the
matrices $D_{t}$ grows with $t$. The next two lemmas show that this
is not the case, and in fact, the right term of
\eqref{covariance_bound} is not growing too fast, which will allow us to
obtain a sub-linear regret bound. \lemref{operator_scalar} analyzes
the properties of the recursion of $D$ defined in
\eqref{D} for scalars, that is $d=1$. In \lemref{eigen_values_lemma} we extend
this analysis to matrices.
\begin{lemma}
\label{operator_scalar}
Define
\(
f(\lambda) = {\lambda \beta}/\paren{\lambda+ \beta} + x^2
\)
for $\beta,\lambda \geq 0$ and some $x^2 \leq \gamma^2$. Then:
\begin{enumerate}
\item $f(\lambda) \leq \beta +
\gamma^2$
\item $f(\lambda) \leq \lambda + \gamma^2$
\item $f(\lambda) \leq \max\braces{\lambda,\frac{3\gamma^2 +
  \sqrt{\gamma^4+4\gamma^2\beta}}{2}}$
\end{enumerate}
\end{lemma}
\begin{proof}
For the first property we have $f(\lambda) = {\lambda
  \beta}/\paren{\lambda+ \beta} + x^2 \leq \beta\times 1 + x^2$.
The second property follows from the symmetry between $\beta$ and
$\lambda$.
To prove the third property we decompose the function as,
\(
f(\lambda) = \lambda - \frac{\lambda^2 }{\lambda+ \beta} + x^2
\).
Therefore, the function is bounded by its argument $f(\lambda)\leq
\lambda$ if, and only if, $- \frac{\lambda^2 }{\lambda+ \beta} + x^2
\leq 0$.
Since we assume $x^2\leq\gamma^2$, the last inequality holds if,
\(
-\lambda^2 + \gamma^2 \lambda + \gamma^2\beta \leq 0
\),
which holds for $\lambda \geq \frac{\gamma^2 +
  \sqrt{\gamma^4+4\gamma^2\beta}}{2}$.

To conclude. If $\lambda \geq \frac{\gamma^2 +
  \sqrt{\gamma^4+4\gamma^2\beta}}{2}$, then $f(\lambda) \! \leq \! \lambda$.
Otherwise, by the second property, we have, $$f(\lambda) \! \leq \! \lambda\!+\!\gamma^2
\! \leq \! \frac{\gamma^2 \!+\!
  \sqrt{\gamma^4\!+\!4\gamma^2\beta}}{2} \!+\! \gamma^2 = \frac{3\gamma^2 \!+\!
  \sqrt{\gamma^4\!+\!4\gamma^2\beta}}{2},$$as required.
\end{proof}

We build on \lemref{operator_scalar} to bound the maximal eigenvalue of the
matrices $D_{t}$.
\begin{lemma}
\label{eigen_values_lemma}
Assume $\normt{\vxi{t}} \leq X^2$ for some $X$. Then, the eigenvalues
of $D_{t}$ (for $t \geq 1$), denoted by $\lambda_i\paren{D_{t}}$, are upper bounded
by
\[\max_i\lambda_i\paren{D_{t}}\leq\max\braces{ \frac{3X^2 +
  \sqrt{X^4+4X^2 c}}{2},b+X^2} ~. \]
\end{lemma}
\begin{proof}
  By induction.  From \eqref{D} we have that
  $\lambda_i(D_{1}) \leq b + X^2$  for
$i=1\comdots d$. We proceed with a proof for some $t$. For simplicity,
denote by $\lambda_i = \lambda_i(D_{t-1})$ the i$th$
eigenvalue of $D_{t-1}$ with a corresponding
eigenvector $\vvi{i}$.
From \eqref{D} we have,
\begin{align}
D_{t}
&=\left(D_{t-1}^{-1}+c^{-1}\mi\right)^{-1}+\vxi{t}\vxti{t}\nonumber\\
& \preceq \left(D_{t-1}^{-1}+c^{-1}\mi\right)^{-1}  +
\mi \normt{\vxi{t}}\nonumber\\
& = \sum_i^d \vvi{i}
\vvti{i}\paren{\paren{\lambda_i^{-1} + c^{-1}}^{-1} + \normt{\vxi{t}}}\nonumber\\
 &= \sum_i^d \vvi{i}
\vvti{i}\paren{ \frac{\lambda_i c }{\lambda_i + c} +
  \normt{\vxi{t}}} ~.\label{bound_eigens}
\end{align}
Plugging  \lemref{operator_scalar} in \eqref{bound_eigens} we get,
\begin{align*}
D_{t}
&\preceq \sum_i^d \vvi{i}
\vvti{i}\max\braces{ \frac{3X^2 +
  \sqrt{X^4+4X^2 c}}{2},b+X^2}\\
&= \max\braces{ \frac{3X^2 +
  \sqrt{X^4+4X^2 c}}{2},b+X^2} \mi~.
\end{align*}
\end{proof}

Finally, equipped with the above lemmas we are able to
prove the main
result of this section.
\begin{corollary}
\label{cor:main_corollary}
Assume $\normt{\vxi{t}}\leq X^2$, $\vert\yi{t}\vert \leq Y$. Then, 
\begin{align}
L_T(\textrm{LASER})\leq 
  b\left\Vert
     \mathbf{u}_{1}\right\Vert ^{2}+L_T(\{\vui{t}\})+Y^{2} \ln\left|\frac{1}{b}\mathbf{D}_{T}\right|\nonumber\\ +c^{-1}Y^2\tr\paren{\mathbf{D}_0}+c V^{(2)}\nonumber\\
 +c^{-1}Y^2 T d  \max\braces{ \frac{3X^2 +
   \sqrt{X^4+4X^2 c}}{2},b+X^2}~.\label{final_cor}
\end{align} 

Furthermore, set
$b=\varepsilon c$ for some $0<\varepsilon<1$.
Denote by 
\(
\mu = \max\braces{9/8X^2, \frac{\paren{b+X^2}^2}{8X^2}}
\) and 
\(M =
\max\braces{3X^2, b+X^2}
\).
If $V^{(2)} \leq T \frac{\sqrt{2}Y^2dX}{\mu^{3/2}}$ (low drift) then by setting 
\begin{align}
c= \frac{\sqrt{2}T Y^2 d X}{\paren{V^{(2)}}^{2/3}}\label{c1}
\end{align}
 we have,
\begin{align}
& L_T(\textrm{LASER})
\leq\nonumber\\ 
& \quad b\left\Vert \mathbf{u}_{1}\right\Vert ^{2} + 3\paren{\sqrt{2}Y^2 d X}^{2/3} T^{2/3} \paren{V^{(2)}}^{1/3}\nonumber\\
&\quad +\frac{\varepsilon}{1-\varepsilon}Y^{2}d +L_T(\{\vui{t}\})
+Y^{2} \ln\left|\frac{1}{b}\mathbf{D}_{T}\right|\label{bound1}~.
\end{align}

\end{corollary}
The proof appears in \secref{proof_cor_main}.
Note that if $V^{(2)} \geq T \frac{Y^2dM}{\mu^{2}}$ then by setting 
\(
c=
 \sqrt{{Y^2dMT}/{V^{(2)}}}
\) 
 we have,
\begin{align}
& L_T(\textrm{LASER})
\leq b\left\Vert \mathbf{u}_{1}\right\Vert ^{2}
 + 2\sqrt{Y^2 d TMV^{(2)}}\nonumber\\
& \quad+\frac{\varepsilon}{1-\varepsilon}Y^{2}d +L_T(\{\vui{t}\})
+Y^{2} \ln\left|\frac{1}{b}\mathbf{D}_{T}\right| \label{high_drift}
\end{align}
(See \secref{details_for_second_bound} for details).
The last bound is linear in $T$ and can be obtained also by a naive
algorithm that outputs $\hat{y}_t=0$ for all $t$.

A few remarks are in
order. When the variance $V^{(2)}=0$ goes to zero, we set
$c=\infty$ and thus we have $D_{t}=b\mi+\sum_{s=1}^t
\vxi{s}\vxti{s}$ used in recent
algorithms~(\cite{Vovk01,Forster,Hayes,CesaBianchiCoGe05}). In this case
the algorithm reduces to the algorithm by Forster~\cite{Forster} (which is also the AAR algorithm of Vovk~\cite{Vovk01}), with
the same logarithmic regret bound (note that the last term in the
bounds is logarithmic in $T$, see the proof of
Forster~\cite{Forster}). See also the work of Azoury and Warmuth~\cite{AzouryWa01}.

%% file: simulations1.tex
\section{Simulations}
\label{sec:simulations}

\begin{figure*}[!t!]
\begin{center}
\subfigure{\includegraphics[width=0.32\textwidth]{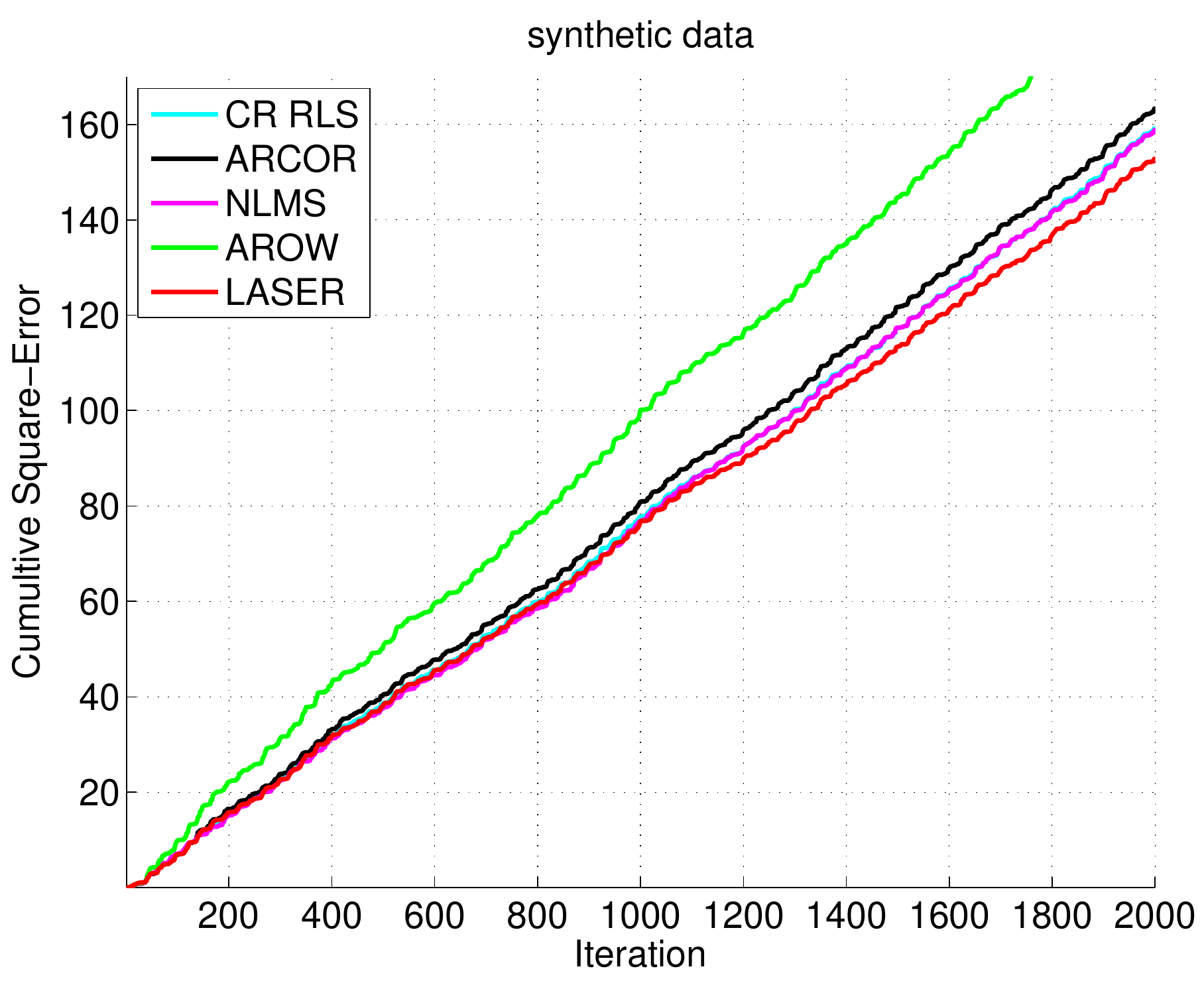}}
\subfigure{\includegraphics[width=0.32\textwidth]{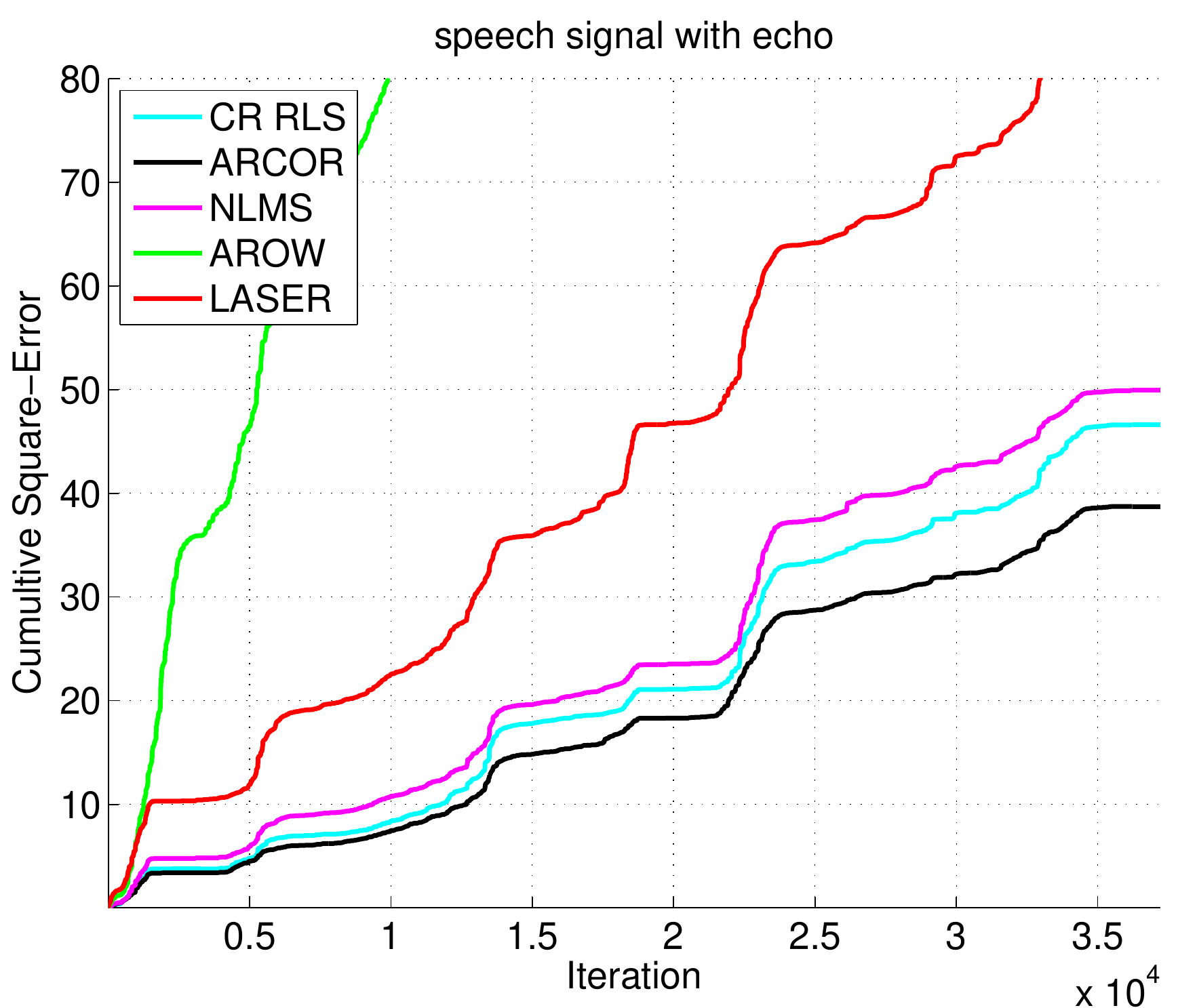}}
\subfigure{\includegraphics[width=0.32\textwidth]{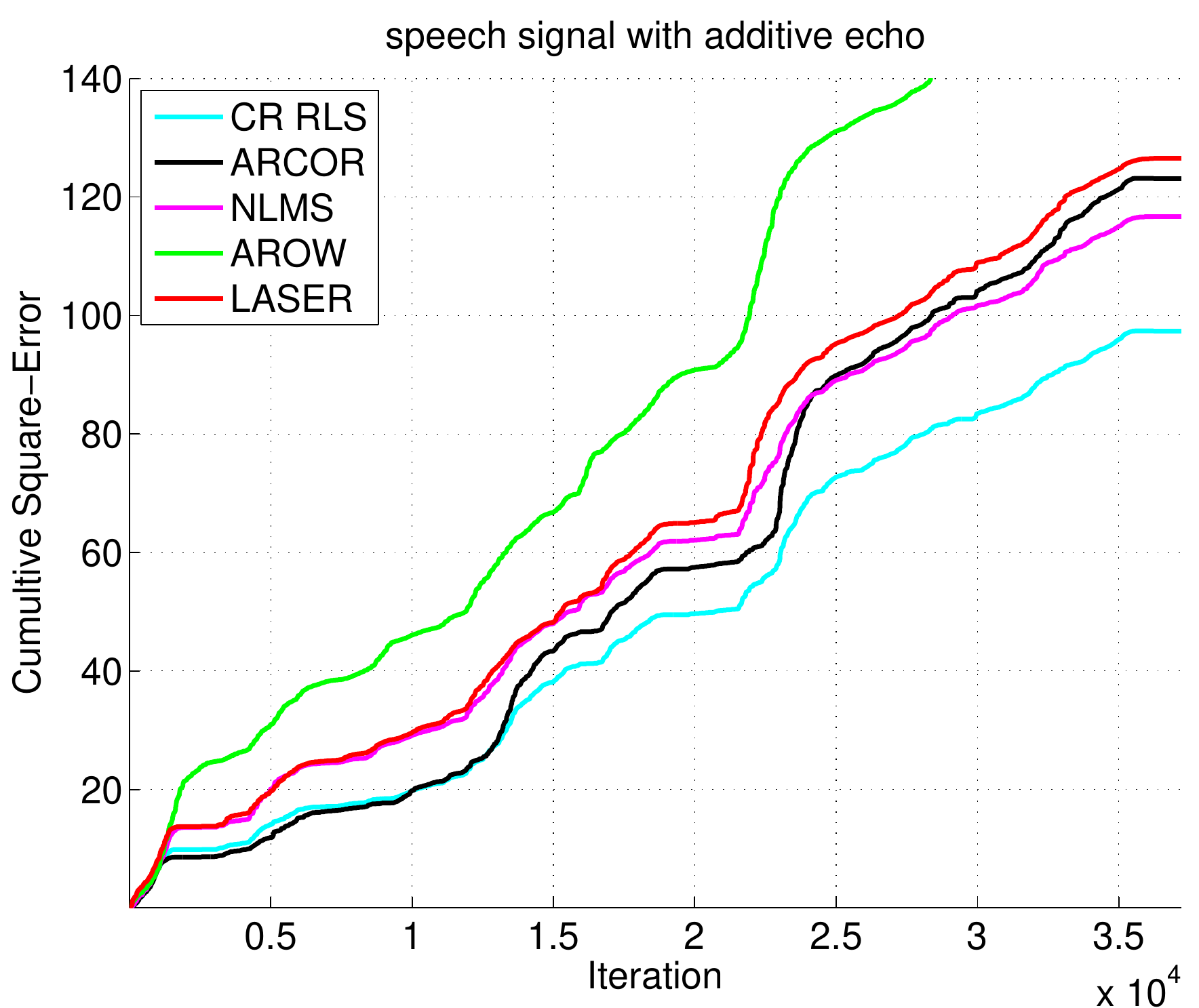}}
\caption{Cumulative squared loss for AROWR, ARCOR, LASER, NLMS and CR-RLS vs iteration. Left panel shows result for synthetic datasets with drift, and two right panels show results for a problem of acoustic echo cancelation on speech signal (best shown in color).}
\label{fig:sims}
\end{center}
\end{figure*}


We evaluate our algorithms on three datasets, one synthetic and two real world.
The synthetic dataset contains $2,000$ points in $\reals^{20}$, where the first ten coordinates were grouped into five groups of size two. Each such pair was drawn from a $45^\circ$ rotated Gaussian distribution with standard deviations $10$ and $1$.  The remaining $10$ coordinates were drawn from independent Gaussian distributions $\norm\paren{0,2}$.  The dataset was generated using a sequence of vectors $\vui{t}\in\reals^{20}$ for which the only non-zero coordinates are the first two, where their values are the coordinates of a unit vector that is rotating with a constant rate. Specifically, we have $\Vert\vui{t}\Vert=1$ and the instantaneous drift $\Vert\vui{t}-\vui{t-1}\Vert$ is constant.  

The other two datasets are generated from echoed speech signal. The first speech echoed signal was generated using FIR filter with $k$ delays and varying attenuated amplitude. This effect imitates acoustic echo reflections from large, distant and dynamic obstacles.  The difference equation $y(n)=x(n)+\sum_{D=1}^k{A(n)x(n-D)}+v(n)$ was used, where $D$ is a delay in samples, the coefficient $A(n)$ describes the changing attenuation related from object reflection and $v(n)\sim\norm\paren{0,10^{-3}}$ is a white noise.  The second speech echoed signal was generated using a flange IIR filter, where the delay is not constat, but changing with time.  This effect imitates time stretching of audio signal caused by moving and changing objects in the room.  The difference equation $y(n)=x(n)+Ay\paren{n-D(n)}+v(n)$ was used.

Five algorithms are evaluated: NLMS (normalized least mean square)~(\cite{Bershad,Bitmead}) which is a state-of-the-art first-order algorithm, AROWR (AROW for Regression) with no restarts nor projection, ARCOR, LASER and CR-RLS.  For the synthetic datasets the algorithms' parameters were tuned using a single random sequence. We repeat each experiment $100$ times reporting the mean cumulative square-loss. We note that AAR~(\cite{vovkAS,Vovk01}) is a special case of LASER and RLS is a special case of CR-RLS, for a specific choice of their respective parameters. Additionally, the performance of AROWR, AAR and RLS is similar, and thus only the performance of AROWR is shown.
%
%
%
For the speech signal the algorithms' parameters were tuned on $10\%$ of the signal, then the best parameter choices for each algorithm were used to evaluate the performance on the remaining signal.

The results are summarized in \figref{fig:sims}.  AROWR performs worst on all datasets as it converges very fast and thus not able to track the changes in the data. Focusing on the left panel, showing the results for the synthetic signal, we observe that ARCOR performs relatively bad as suggested by our analysis for constant, yet not too large, drift. Both CR-RLS and NLMS perform better, where CR-RLS is slightly better as it is a second order algorithm, and allows to converge faster between switches. On the other hand, NLMS is not converging and is able to adapt to the drift. Finally, LASER performs the best, as hinted by its analysis, for which the bound is lower where there is a constant drift.

Moving to the center panel, showing the results for first echoed speech signal with varying amplitude, we observe that LASER is the worst among all algorithms except AROWR. Indeed, it does preventing convergence by keeping the learning rates far from zero, yet it is a min-max algorithm designed for the worst-case, which is not the case for real-world speech data. However, speech data is highly regular and the instantaneous drift vary. NLMS performs better as it is not converging, yet both CR-RLS and ARCOR perform even better, as they both not-converging due to covariance resets on the one hand, and second order updates on the other hand. ARCOR outperforms CR-RLS as the former adapts the resets to actual data, and is not using pre-defined scheduling as the later.

Finally, the right panel summarizes the results for evaluations on the second echoed speech signal. Note that the amount of drift grows since the data is generated using
flange filter. 
Both LASER and ARCOR are out-performed as both assume drift that is sublinear or at most linear, which is not the case. CR-RLS outperforms NLMS. The later is first order, so is able to adapt to changes, yet has slower converge rate. The former is able to cope with drift due to resets.

Interestingly, in all experiments, NLMS was not performing the best nor the worst. There is no clear winner among the three algorithms that are both second order, and designed to adapt to drifts. Intuitively, if the drift suits the assumptions of an algorithm, that algorithm would perform the best, and otherwise, its performance may even be worse than of NLMS.


\begin{figure}
\begin{center}
\includegraphics[width=0.4\textwidth]{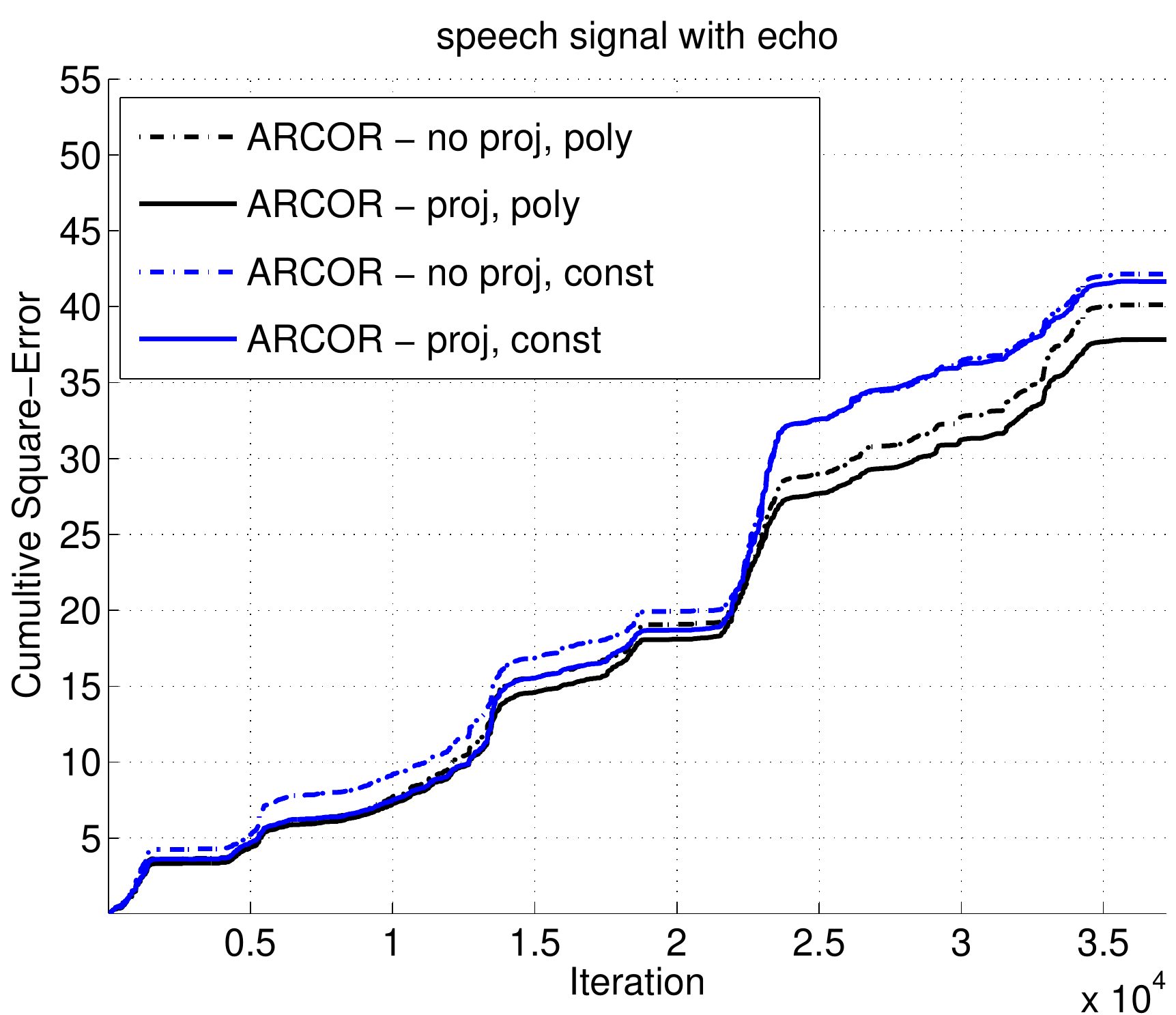}
\includegraphics[width=0.4\textwidth]{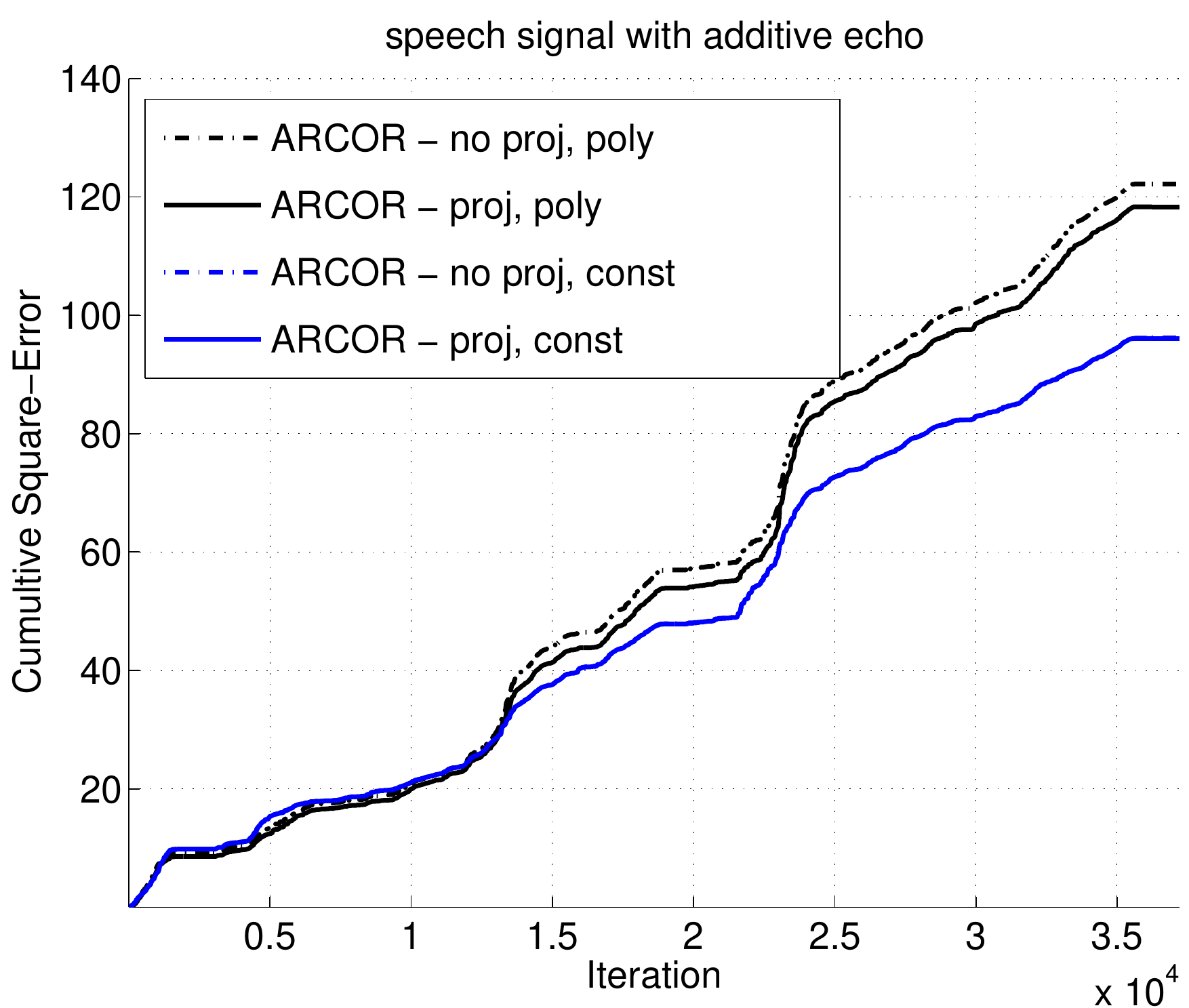}
\end{center}
\caption{Cumulative squared loss of four variants of ARCOR vs iteration.}
\label{fig:sims_arcor}
\end{figure}
We have seen above that ARCOR performs a projection step, which partially was motivated from the analysis. We now evaluate its need and affect in practice on two speech problems. We test two modifications of ARCOR, resulting in four variants altogether. First, we replace the the polynomial thresholds scheme to the constant thresholds scheme, that is, all thresholds are equal. Second, we omit the projection step.  The results are summarized in \figref{fig:sims_arcor}. The line corresponding to the original algorithm, is called ``proj, poly'' as it performs a projection step and uses polynomial schema for the lower-bound on eigenvalues. The version that omits projection and uses constant schema, called ``no proj, const'', is most similar to CR-RLS. Both resets the covariance matrix, CR-RLS after fixed amount of iterations, while
``ARCOR-no proj, const'' when the eigenvalues meets a specified fixed lower bound.
The difference between the two plots is the amount of drift used: the top panel shows results for sublinear drift, and the bottom panel shows results with increasing per-instance drift. The original version, as hinted by the analysis, is designed to work with sub-linear drift, and performs the best in this case. However, when this assumption over the amount of drift breaks, this version is not optimal anymore, and constat schema performs better, as it allows the algorithm to adapt to non-vanishing drift. Finally, in both datasets, the algorithm that perform best perform a projection step after each iteration. Providing some empirical evidence for its need.

%% file: summary.tex
\section{Summary and Conclusions}
\label{sec:summary_conclusions}
We proposed and analyzed two novel algorithms for non-stationary
online regression designed and analyzed with the squared loss in the
worst-case regret framework. The ARCOR algorithm was built on AROWR,
that employs second order information, yet performs data-dependent
covariance resets, which provides it the ability to track drifts. The
LASER algorithm was built on the last-step minmax predictor with the
proper modifications for non-stationary problems. Regret bounds shown
and proven are optimized using knowledge on the amount of drift, and
in general the two algorithms are not comparable.

Few open directions are possible. First, to extend these algorithms to
other loss functions rather than the squared loss. Second, currently,
direct implementation of both algorithms requires either matrix
inversion or eigenvector decomposition. A possible direction is
to design a more efficient version of these algorithms. Third, an
interesting direction is to design algorithms that automatically
detect the level of drift, or do not need this information before run-time.

%% file: appendix.tex
\subsection{Proof of \lemref{lem:lemma11}}
\label{proof_lemma11}
\begin{proof}
We calculate

\begin{align*}
P_{t}\left(\vui{t}\right)
=&  \min_{\vui{1},\ldots,\vui{t-1}} \Bigg(b\left\Vert \vui{1}\right\Vert ^{2}+c\sum_{s=1}^{t-1}\left\Vert \vui{s+1}-\vui{s}\right\Vert ^{2}\\
&+\sum_{s=1}^{t}\left(y_{s}-\vuti{s}\vxi{s}\right)^{2}\Bigg)\\
 =&   \min_{\vui{1},\ldots,\vui{t-1}}
\Bigg(b\left\Vert \vui{1}\right\Vert
^{2}+c\sum_{s=1}^{t-2}\left\Vert
  \vui{s+1}-\vui{s}\right\Vert
^{2}\\
&+\sum_{s=1}^{t-1}\left(y_{s}-\vuti{s}\vxi{s}\right)^{2}
+c\left\Vert \vui{t}-\vui{t-1}\right\Vert ^{2}\\
&+\left(y_{t}-\vuti{t}\vxi{t}\right)^{2}\Bigg)\\
  =&
  \min_{\vui{t-1}}\min_{\vui{1},\ldots,\vui{t-2}}\Bigg(b\left\Vert
    \vui{1}\right\Vert ^{2}+c\sum_{s=1}^{t-2}\left\Vert
    \vui{s+1}-\vui{s}\right\Vert
  ^{2}\\
&+\sum_{s=1}^{t-1}\left(y_{s}-\vuti{s}\vxi{s}\right)^{2}
+c\left\Vert \vui{t}-\vui{t-1}\right\Vert ^{2}\\
&+\left(y_{t}-\vuti{t}\vxi{t}\right)^{2}\Bigg)\\
  =&
 \min_{\vui{t-1}}\Bigg[\min_{\vui{1},\ldots,\vui{t-2}}
 \bigg(b\left\Vert
     \vui{1}\right\Vert ^{2}+c\sum_{s=1}^{t-2}\left\Vert
     \vui{s+1}-\vui{s}\right\Vert
   ^{2}\\
&+\sum_{s=1}^{t-1}\left(y_{s}-\vuti{s}\vxi{s}\right)^{2}\bigg)\\
&+c\left\Vert \vui{t}-\vui{t-1}\right\Vert ^{2}+\left(y_{t}-\vuti{t}\vxi{t}\right)^{2}\Bigg]\\
   =&  \min_{\vui{t-1}}\Bigg(P_{t-1}\left(\vui{t-1}\right)+c\left\Vert \vui{t}-\vui{t-1}\right\Vert ^{2}\\
&+\left(y_{t}-\vuti{t}\vxi{t}\right)^{2}\Bigg) ~.
\end{align*}
\end{proof}

\subsection{Proof of \lemref{lem:lemma12}}
\label{proof_lemma12}
\begin{proof}
By definition,
\begin{align*}
P_{1}\left(\vui{1}\right) & = Q_{1}\left(\vui{1}\right)
 = b\left\Vert \vui{1}\right\Vert ^{2}+\left(y_{1}-\vuti{1}\vxi{1}\right)^{2}
 \\
& =\vuti{1}\left(b\mi+\vxi{1}\vxti{1}\right)\vui{1}-2y_{1}\vuti{1}\vxi{1}+y_{1}^{2} ~,
\end{align*}
and indeed
\(
D_{1}=b\mi+\vxi{1}\vxti{1}
\),
\(
\vei{1}=y_{1}\vxi{1}
\), and
\(
f_{1}=y_{1}^{2}
\).

We proceed by induction, assume that,
\(
P_{t-1}\left(\vui{t-1}\right)=\vuti{t-1}D_{t-1}\vui{t-1}-2\vuti{t-1}\vei{t-1}+f_{t-1}
\).

Applying \lemref{lem:lemma11} we get,
\begin{align*}
P_{t}\left(\vui{t}\right) &= \min_{\vui{t-1}}\Bigg(\vuti{t-1}D_{t-1}\vui{t-1}-2\vuti{t-1}\vei{t-1}+f_{t-1}\\
&\quad +c\left\Vert \vui{t}-\vui{t-1}\right\Vert ^{2}+\left(y_{t}-\vuti{t}\vxi{t}\right)^{2}\Bigg)\\
&  = \min_{\vui{t-1}}\Bigg(\vuti{t-1}\left(c\mi+D_{t-1}\right)\vui{t-1}-2\vuti{t-1}
\left(c\vui{t}+\vei{t-1}\right)\\
&\quad +f_{t-1}+c\left\Vert \vui{t}\right\Vert ^{2}+\left(y_{t}-\vuti{t}\vxi{t}\right)^{2}\Bigg)\\
&  =  -\left(c\vui{t}+\vei{t-1}\right)^{\top}\left(c\mi+D_{t-1}\right)^{-1}\left(c\vui{t}+\vei{t-1}\right)
+f_{t-1}\\
&\quad +c\left\Vert \vui{t}\right\Vert ^{2}+\left(y_{t}-\vuti{t}\vxi{t}\right)^{2}\\
&  = \vuti{t}\left(c\mi+\vxi{t}\vxti{t}-c^{2}\left(c\mi+D_{t-1}\right)^{-1}\right)\vui{t}\\
& \quad -2\vuti{t}\left[c\left(c\mi+D_{t-1}\right)^{-1}\vei{t-1}+y_{t}\vxi{t}\right]\\
& \quad -\veti{t-1}
 \left(c\mi+D_{t-1}\right)^{-1}\vei{t-1}+f_{t-1}+y_{t}^{2}~.
\end{align*}
Using the Woodbury identity we continue to develop the last equation,
\begin{align*}
& = \vuti{t}\left(c\mi+\vxi{t}\vxti{t}-c^{2}\left[c^{-1}\mi-c^{-2}\left(D_{t-1}^{-1}+c^{-1}
\mi\right)^{-1}\right]\right)\vui{t}\\
 & \quad  -2\vuti{t}\left[\left(\mi+c^{-1}D_{t-1}\right)^{-1}\vei{t-1}+y_{t}\vxi{t}\right]\\
 & \quad -\veti{t-1}
 \left(c\mi+D_{t-1}\right)^{-1}\vei{t-1}+f_{t-1}+y_{t}^{2}\\
 & =  \vuti{t}\left(\left(D_{t-1}^{-1}+c^{-1}\mi\right)^{-1}+\vxi{t}\vxti{t}\right)\vui{t}\\
 & \quad -2\vuti{t}\left[\left(\mi+c^{-1}D_{t-1}\right)^{-1}\vei{t-1}+y_{t}\vxi{t}\right]\\
 & \quad -\veti{t-1}
 \left(c\mi+D_{t-1}\right)^{-1}\vei{t-1}+f_{t-1}+y_{t}^{2}~,
\end{align*}
and indeed
$
D_{t}=\left(D_{t-1}^{-1}+c^{-1}\mi\right)^{-1}+\vxi{t}\vxti{t}
$,
$
\vei{t}=\left(\mi+c^{-1}D_{t-1}\right)^{-1}\vei{t-1}+y_{t}\vxi{t}
$ and,
$
f_{t}=f_{t-1}-\veti{t-1}\left(c\mi+D_{t-1}\right)^{-1}\vei{t-1}+y_{t}^{2}
$, as desired.
\end{proof}

\subsection{Proof of Theorem~\ref{non_stat_bound}}
\label{sec:ARCOR_Theoreme_proof}
We prove the theorem in four steps. First, we state the following
technical lemma, for which we define the following
notation,
\begin{align*}
&\dta{\vz}{\vv}{t} = (\vz-\vv)^\top \msigmai{t}^{-1} (\vz-\vv) ~,~\\
&\dta{\vz}{\vv}{\tilde{t}} = (\vz-\vv)^\top
\tilde{\msigma}_{t}^{-1} (\vz-\vv) ~,~\\
&\chit{t} = \vxi{t}^\top \msigmai{t-1}\vxii~.
\end{align*}
Second, we define a telescopic sum and in \lemref{deltas} prove a lower
bound for each element. Third, in \lemref{log_det} we upper bound one term of
the telescopic sum, and finally, in the fourth step we combine all
these parts to conclude the proof. Let us start with the technical lemma.
\begin{lemma}
\label{reg_lemma}
Let ${\tvwi{t}}$ and $\tilde{\msigma}_{t}$ be defined in
\eqref{sigma_alg} and
\eqref{mu_alg}, 
then,
\[
\dta{\vwi{t-1}}{\vui{t-1}}{t-1}-\dta{\tilde{\vwi{t}}}{\vui{t-1}}{\tilde{t}}
= \frac{1}{\cor}\ellt{t} - \frac{1}{\cor} \gllt{t} -
\frac{\ellt{t}\chit{t}}{\cor\paren{\cor+\chit{t}}}~.
\]
where $\ellt{t} = \paren{\yi{t}-\vwti{t-1}\vxi{t}}^2$ and $\gllt{t} = \paren{\yi{t}-\vuti{t-1}\vxi{t}}^2$.
\label{lemma:derivation_of_bound}
\end{lemma}
\begin{proof}
We start by writing the distances explicitly
\begin{align*}
  & \dta{\vwi{t-1}}{\vui{t-1}}{t-1}-\dta{\tilde{\vw}_t}{\vui{t-1}}{\tilde{t}}\\
  = & -\tran{\paren{\vui{t-1} - \tilde{\vw}_{t}}} \tilde{\msigma}_{t}^{-1}\paren{\vui{t-1} - \tilde{\vw}_t} \\&+ \tran{\paren{\vui{t-1} - \vwi{t-1}}} \msigmai{t-1}^{-1}\paren{\vui{t-1} - \vwi{t-1}}~.
\end{align*}
Substituting ${\tvwi{t}}$ as appears in \eqref{mu_alg} the last
equation becomes,
\begin{align*}
 & -\tran{\paren{\vui{t-1} - \vwi{t-1}}}
 \tilde{\msigma}_{t}^{-1}\paren{\vui{t-1} - \vwi{t-1}}
\\&+2(\vui{t-1}-\vwi{t-1})\tilde{\msigma}_{t}^{-1}
  \msigmai{t-1}\vxi{t}\frac{(\yi{t}-\vxti{t}\vwi{t-1})}{\cor+\vxti{t}\msigmai{t-1}\vxi{t}} \\
  &  - \paren{\frac{(\yi{t}-\vxti{t}\vwi{t-1})}{\cor+\vxti{t}\msigmai{t-1}\vxi{t}}}^2 \vxti{t}\msigmai{t-1}\tilde{\msigma}_{t}^{-1}\msigmai{t-1}\vxi{t} \\&+ \tran{\paren{\vui{t-1} - \vwi{t-1}}} \msigmai{t-1}^{-1}\paren{\vui{t-1} - \vwi{t-1}}~.
\end{align*}
Plugging $\tilde{\msigma}_{t}$ as appears in \eqref{sigma_alg} we get,
\begin{align*}
& \dta{\vwi{t-1}}{\vui{t-1}}{t-1}-\dta{\tilde{\vw}_t}{\vui{t-1}}{\tilde{t}}\\
  = & -\tran{\paren{\vui{t-1} - \vwi{t-1}}} \paren{\msigmai{t-1}^{-1}+\frac{1}{\cor}\vxi{t}\vxti{t}}\paren{\vui{t-1} - \vwi{t-1}}\\
  &  + 2\tran{(\vui{t-1}\!-\!\vwi{t-1})}\!\!\paren{\msigmai{t-1}^{-1}\!+\!\frac{1}{\cor}\vxi{t}\vxti{t}}\msigmai{t-1}\!\vxi{t}\frac{(\yi{t}\!-\!\vxti{t}\vwi{t-1})}{\cor\!+\!\vxti{t}\msigmai{t-1}\vxi{t}} \\
  &  - \frac{(\yi{t}-\vxti{t}\vwi{t-1})^2}{\paren{\cor+\vxti{t}\msigmai{t-1}\vxi{t}}^2}\vxti{t}\msigmai{t-1}
  \paren{\msigmai{t-1}^{-1} + \frac{1}{\cor}\vxi{t}\vxti{t}}\msigmai{t-1}\vxi{t}\\
  &  + \tran{\paren{\vui{t-1} - \vwi{t-1}}} \msigmai{t-1}^{-1}\paren{\vui{t-1} - \vwi{t-1}}~.
\end{align*}
Finally, we substitute $\ellt{t} = \paren{\yi{t}-\vxti{t}\vwi{t-1}}^2$
, $\gllt{t} = \paren{\yi{t}-\vxti{t}\vui{t-1}}^2$ and, $\chit{t} =
\vxti{t}\msigmai{t-1}\vxi{t}$. Rearranging the terms,
\begin{align*}
& \dta{\vwi{t-1}}{\vui{t-1}}{t-1}-\dta{\tilde{\vw}_t}{\vui{t-1}}{\tilde{t}}\\
  = & - \frac{1}{\cor}\paren{\yi{t}-\vxti{t}\vwi{t-1} - \paren{\yi{t}-\vxti{t}\vui{t-1}}}^2\\
  & -\frac{2\paren{\yi{t}\!-\!\vxti{t}\vui{t-1}\!-\!\paren{\yi{t}\!-\!\vxti{t}\vwi{t-1}}}\paren{\yi{t}\!-\!\vxti{t}\vwi{t-1}}}{\cor+\chit{t}}\!\paren{1\!+\!\frac{\chit{t}}{\cor}}\\
  &  - \frac{\ellt{t}\chit{t}}{\paren{\cor+\chit{t}}^2}\paren{1+\frac{\chit{t}}{\cor}}\\
  = & - \frac{1}{\cor}\ellt{t}
  + 2\paren{\yi{t}-\vxti{t}\vwi{t-1}}\paren{\yi{t}-\vxti{t}\vui{t-1}}\frac{1}{\cor} - \frac{1}{\cor} \gllt{t}\\&
  + \frac{2\ellt{t}}{\cor+\chit{t}}\paren{1+\frac{\chit{t}}{\cor}} - \frac{\ellt{t}\chit{t}}{\cor\paren{\cor+\chit{t}}}\\
  &  -2\frac{\paren{\yi{t}-\vxti{t}\vwi{t-1}}\paren{\yi{t}-\vxti{t}\vui{t-1}}}{\cor+\chit{t}}\paren{1+\frac{\chit{t}}{\cor}}\\
  = & \frac{1}{\cor}\ellt{t} - \frac{1}{\cor} \gllt{t} -
  \frac{\ellt{t}\chit{t}}{\cor\paren{\cor+\chit{t}}} ~,
\end{align*}
which completes the proof.
\end{proof}
We now define one element of the telescopic sum, and lower bound it,
\begin{lemma}\label{deltas}
Denote by
\[
\Delta_t = \dta{\vwi{t-1}}{\vui{t-1}}{t-1}
-\dta{\vwi{t}}{\vui{t}}{t}
\]
then
\begin{align*}
  \Delta_t & \geq \frac{1}{\cor}\paren{\ellt{t} - \gllt{t}} -
  \ellt{t} \frac{\chit{t}}{\cor(\cor+\chit{t})}\nonumber\\
  & \quad + \vuti{t-1}\msigmai{t-1}^{-1}\vui{t-1} -
  \vuti{t}\msigmai{t}^{-1}\vui{t} -
  2\rb\Lambda_i^{-1}\Vert\vui{t-1}-\vui{t}\Vert
\end{align*}
where $i-1$ is the number of restarts occurring before example $t$.
\end{lemma}
\begin{proof}
We write $\Delta_t$ as a telescopic sum of four terms as follows,
\begin{align*}
 \Delta_{t,1} &= \dta{\vwi{t-1}}{\vui{t-1}}{t-1}-\dta{\tilde{\vw}_t}{\vui{t-1}}{\tilde{t}}\\
 \Delta_{t,2} &= \dta{\tilde{\vw}_t}{\vui{t-1}}{\tilde{t}}-\dta{\tilde{\vw}_t}{\vui{t-1}}{t}\\
\Delta_{t,3} &=\dta{\tilde{\vw}_t}{\vui{t-1}}{t}-\dta{\vwi{t}}{\vui{t-1}}{t}\\
\Delta_{t,4} &= \dta{\vwi{t}}{\vui{t-1}}{t}-\dta{\vwi{t}}{\vui{t}}{t}
\end{align*}

We lower bound each of the four terms. Since the value of
$\Delta_{t,1}$ was computed in \lemref{reg_lemma}, we start with the second
term. If no reset occurs then $\msigmai{t}=\tilde{\msigma}_t$ and
$\Delta_{t,2} = 0$. Otherwise, we use the facts that $0 \preceq\tilde{\msigmai{t}} \preceq\mi$ and $\msigmai{t} = \mi$, and get,
\begin{align*}
\Delta_{t,2} &= \tran{\paren{\tilde{\vw}_t-\vui{t-1}}}\tilde{\msigma}_t^{-1}\paren{\tilde{\vw}_t-\vui{t-1}} \\
  & \quad - \tran{\paren{\tilde{\vw}_t-\vui{t-1}}}\msigmai{t}^{-1}\paren{\tilde{\vw}_t-\vui{t-1}}\\
  & \geq \tr\paren{{\paren{\tilde{\vw}_t-\vui{t-1}}\tran{\paren{\tilde{\vw}_t-\vui{t-1}}}
  \paren{\mi-\mi}}} = 0 ~.
\end{align*}
To summarize, $\Delta_{t,2} \geq 0$.
We can lower bound $\Delta_{t,3} \geq 0$ by using the fact that
$\vwii$ is a projection of $\tilde{\vw}_t$ onto a closed set (a ball
of radius $\rb$ around the origin), which by our assumption contains
$\vui{t}$. Employing Corollary~3 of Herbster and Warmuth~\cite{HerbsterW01} we get,
$\dta{\tilde{\vw}_t}{\vui{t-1}}{t} \geq \dta{\vwi{t}}{\vui{t-1}}{t}$
and thus  $\Delta_{t,3} \geq 0$.

Finally, we lower bound for the fourth term $\Delta_{t,4}$,
\begin{align}
  \Delta_{t,4} &= \tran{\paren{\vwi{t}-\vui{t-1}}}\msigmai{t}^{-1}\paren{\vwi{t}-\vui{t-1}}\nonumber\\
  & \quad - \tran{\paren{\vwi{t}-\vui{t}}}\msigmai{t}^{-1}\paren{\vwi{t}-\vui{t}}\label{lower_bound_delta_4}\\
  & = \vuti{t-1}\msigmai{t}^{-1}\vui{t-1} -
  \vuti{t}\msigmai{t}^{-1}\vui{t} -
  2\vwti{t}\msigmai{t}^{-1}\paren{\vui{t-1}-\vui{t}} \nonumber
\end{align}
We use H\"{o}lder inequality and then Cauchy-Schwartz inequality to get the
following lower bound,
\begin{align*}
  &-2\vwti{t}\msigmai{t}^{-1}\paren{\vui{t-1}-\vui{t}}
 =-2
\tr\paren{\msigmai{t}^{-1}\paren{\vui{t-1}-\vui{t}} \vwti{t}}\\
 &~~~~~~~ \geq -2 \lambda_{max} \paren{\msigmai{t}^{-1} }
  \vwti{t}\paren{\vui{t-1}-\vui{t}}\\
 &~~~~~~~ \geq -2 \lambda_{max} \paren{\msigmai{t}^{-1} } \Vert \vwi{t}
  \Vert \Vert{\vui{t-1}-\vui{t}}\Vert~.
\end{align*}
Using the facts that $\Vert\vwii\Vert\leq\rb$ and that
$\lambda_{max} \paren{\msigmai{t}^{-1} }  =
1/\lambda_{min} \paren{\msigmai{t}} \leq \Lambda_{i}^{-1}$, where $i$
is the current segment index, we get,
\begin{align}
  -2\vwti{t}\msigmai{t}^{-1}\paren{\vui{t-1}-\vui{t}}
    &\geq -2  \Lambda_{i}^{-1} \rb \Vert{\vui{t-1}-\vui{t}}\Vert\label{lower_bound_delta_4_a}~.
\end{align}
Substituting \eqref{lower_bound_delta_4_a} in
\eqref{lower_bound_delta_4} and using $\msigmaii \preceq
\msigmai{t-1}$ a lower bound is obtained,
 \begin{align}
\Delta_{t,4}  \geq&~ \vuti{t-1}\msigmai{t}^{-1}\vui{t-1} - \vuti{t}\msigmai{t}^{-1}\vui{t} - 2\rb\Lambda_i^{-1}\Vert\vui{t-1}-\vui{t}\Vert\nonumber\\
   \geq&~ \vuti{t-1}\msigmai{t-1}^{-1}\vui{t-1} -
  \vuti{t}\msigmai{t}^{-1}\vui{t} \nonumber\\&-
  2\rb\Lambda_i^{-1}\Vert\vui{t-1}-\vui{t}\Vert ~.
\label{lower_bound_delta_4_final}
\end{align}
Combining \eqref{lower_bound_delta_4_final} with
\lemref{lemma:derivation_of_bound} concludes the proof.
\end{proof}
Next we state an upper bound that will appear in one of the summands
of the telescopic sum,
\begin{lemma}\label{log_det}
During the runtime of the ARCOR algorithm we
have,
\begin{align*}
\sum_{t=t_i}^{t_i+T_i}\! \frac{\chit{t}}{(\chit{t}+\cor)} \!\leq\!
\log{\paren{\det{\paren{\msigmai{t_{i+1}-1}^{-1}}}}} \!=\! \log{\paren{\det{\paren{\paren{\msigma^{i}}^{-1}}}}}~.
\end{align*}
We remind the reader that $t_i$ is the first example index after the
$i$th restart, and $T_i$ is the number of examples observed before the
next restart. We also remind the reader the notation $\msigma^{i} =
\msigmai{t_{i+1}-1}$ is the covariance matrix just before the next restart.
\end{lemma}
The proof of the lemma is similar to the proof of Lemma~4 by Crammer
et.~al.~\cite{CrammerKuDr09} and thus omitted.  We now put all the
pieces together and prove \thmref{non_stat_bound}.
\begin{proof}
We bound the sum $\sum_t\Delta_t$ from above and below, and start with
an upper bound using the property of telescopic sum,
\begin{align}
 \sum_t \Delta_t &= \sum_t \paren{\dta{\vwi{t-1}}{\vui{t-1}}{t-1}
-\dta{\vwi{t}}{\vui{t}}{t}} \nonumber\\
&= \dta{\vwi{0}}{\vui{0}}{0}
-\dta{\vwi{T}}{\vui{T}}{T} \leq {\dta{\vwi{0}}{\vui{0}}{0}}~.
\label{lower_bound}
\end{align}

We compute a lower bound by applying \lemref{deltas},
\begin{align}
&\sum_t \Delta_t
\geq \sum_t  \Bigg(\frac{1}{\cor}\paren{\ellt{t} - \gllt{t}} -
  \ellt{t} \frac{\chit{t}}{\cor(\cor+\chit{t})}\nonumber\\
  & \quad  + \vuti{t-1}\msigmai{t-1}^{-1}\vui{t-1} -
  \vuti{t}\msigmai{t}^{-1}\vui{t} -
  2\rb\Lambda_{i(t)}^{-1}\Vert\vui{t-1}-\vui{t}\Vert\Bigg)~,\nonumber
\end{align}
where ${i(t)}$ is the number of restarts occurred before observing the
$t$th example. Continuing to develop the last equation we obtain,
\begin{align}
\sum_t \Delta_t\geq& \frac{1}{\cor}\sum_t  \ellt{t} - \frac{1}{\cor}\sum_t  \gllt{t} -
  \sum_t\ellt{t} \frac{\chit{t}}{\cor(\cor+\chit{t})}\nonumber\\
  &  + \sum_t \Bigg( \vuti{t-1}\msigmai{t-1}^{-1}\vui{t-1} -
  \vuti{t}\msigmai{t}^{-1}\vui{t}\Bigg)\nonumber\\
  & - \sum_t
  2\rb\Lambda_{i(t)}^{-1}\Vert\vui{t-1}-\vui{t}\Vert\nonumber\\
=& \frac{1}{\cor}\sum_t  \ellt{t} - \frac{1}{\cor}\sum_t  \gllt{t} -
  \sum_t\ellt{t} \frac{\chit{t}}{\cor(\cor+\chit{t})}\nonumber\\
  &  + \vuti{0}\msigmai{0}^{-1}\vui{0} -
  \vuti{T}\msigmai{T}^{-1}\vui{T}\nonumber\\
  & - 2\rb\sum_t
  \Lambda_{i(t)}^{-1}\Vert\vui{t-1}-\vui{t}\Vert ~.\label{upper_bound}
\end{align}

Combining \eqref{lower_bound} with \eqref{upper_bound} and using
$ {\dta{\vwi{0}}{\vui{0}}{0}} =  \vuti{0}\msigmai{0}^{-1}\vui{0}$
(as $\vwi{0}=\vzero$),
\begin{align*}
& \frac{1}{\cor}\sum_t  \ellt{t} - \frac{1}{\cor}\sum_t  \gllt{t} -
  \sum_t\ellt{t} \frac{\chit{t}}{\cor(\cor+\chit{t})} -
  \vuti{T}\msigmai{T}^{-1}\vui{T}\\
& \quad - 2\rb\sum_t
  \Lambda_{i(t)}^{-1}\Vert\vui{t-1}-\vui{t}\Vert \leq 0~.
\end{align*}
Rearranging the terms of the last inequality,
\begin{align*}
\sum_t &  \ellt{t} \leq \sum_t  \gllt{t} +  \sum_t\ellt{t}
\frac{\chit{t}}{\cor+\chit{t}} + \cor  \vuti{T}\msigmai{T}^{-1}\vui{T}\\
& + 2\rb\cor\sum_t
  \frac{1}{\Lambda_{i(t)}}\Vert\vui{t-1}-\vui{t}\Vert~.
\end{align*}

Since $\Vert\vwi{t}\Vert\leq\rb$ and we assume that
$\Vert\vxi{t}\Vert=1$ and $\sup_t \vert\yi{t}\vert=Y$, we get that
$\sup_{ t } \ellt{t} \leq 2(\rb^2+Y^2)$. Substituting the last
inequality in \lemref{log_det}, we bound the second term in the right-hand-side,
\begin{align*}
\sum_t\ellt{t}
\frac{\chit{t}}{\cor+\chit{t}}  &= \sum_i^n \sum_{t=t_i}^{t_i+T_i} \ellt{t}
\frac{\chit{t}}{\cor+\chit{t}} \nonumber\\
&\leq  \sum_i^n \paren{\sup_{ t} \ellt{t}}  \log\det\paren{\paren{\msigma^i}^{-1}} \nonumber\\
&\leq  2\paren{\rb^2+Y^2}\sum_i^n
\log\det\paren{\paren{\msigma^i}^{-1}} ~.
\end{align*}
which completes the proof.
\end{proof}

\subsection{Proof of \lemref{lem:technical}}

\label{proof_lemma_technical}
\begin{proof}
We first use the Woodbury equation to get the following two identities
\begin{align*}
&D_{t}^{-1}=\left[\left(D_{t-1}^{-1}+c^{-1}\mi\right)^{-1}+\vxi{t}\vxti{t}\right]^{-1}\\
& \quad =D_{t-1}^{-1}+c^{-1}\mi-\frac{\left(D_{t-1}^{-1}+c^{-1}\mi\right)\vxi{t}\vxti{t}
\left(D_{t-1}^{-1}+c^{-1}\mi\right)}{1+\vxti{t}\left(D_{t-1}^{-1}+c^{-1}\mi\right)\vxi{t}}\\
&\left(\mi+c^{-1}D_{t-1}\right)^{-1}=\mi-c^{-1}\left(D_{t-1}^{-1}+c^{-1}\mi\right)^{-1}~.
\end{align*}
Multiplying both identities with each other we get, 
\begin{align}
 & D_{t}^{-1}\left(\mi+c^{-1}D_{t-1}\right)^{-1}\nonumber\\
 =~&  \left[D_{t-1}^{-1}+c^{-1}\mi-\frac{\left(D_{t-1}^{-1}+c^{-1}\mi\right)\vxi{t}\vxti{t}\left(D_{t-1}^{-1}+c^{-1}\mi\right)}{1+\vxti{t}\left(D_{t-1}^{-1}+c^{-1}\mi\right)\vxi{t}}\right]\nonumber\\
 &\left[\mi-c^{-1}\left(D_{t-1}^{-1}+c^{-1}\mi\right)^{-1}\right]=\nonumber\\
 =~& D_{t-1}^{-1}-\frac{\left(D_{t-1}^{-1}+c^{-1}\mi\right)\vxi{t}\vxti{t}D_{t-1}^{-1}}{1+\vxti{t}\left(D_{t-1}^{-1}+c^{-1}\mi\right)\vxi{t}}~,\label{identity1}
\end{align}
and, similarly, we multiply the identities in the other order and get,
\begin{align}
 & \left(\mi+c^{-1}D_{t-1}\right)^{-1}D_{t}^{-1}=\nonumber\\
 =~ & \left[\mi-c^{-1}\left(D_{t-1}^{-1}+c^{-1}\mi\right)^{-1}\right]\left[D_{t-1}^{-1}+c^{-1}\mi\vphantom{\frac{\left(D_{t-1}^{-1}+c^{-1}\mi\right)\vxi{t}\vxti{t}\left(D_{t-1}^{-1}+c^{-1}\mi\right)}{1+\vxti{t}\left(D_{t-1}^{-1}+c^{-1}\mi\right)\vxi{t}}}\right.\nonumber\\
 ~& \left.-\frac{\left(D_{t-1}^{-1}+c^{-1}\mi\right)\vxi{t}\vxti{t}\left(D_{t-1}^{-1}+c^{-1}\mi\right)}{1+\vxti{t}\left(D_{t-1}^{-1}+c^{-1}\mi\right)\vxi{t}}\right]\nonumber\\
 =~ &
 D_{t-1}^{-1}-\frac{D_{t-1}^{-1}\vxi{t}\vxti{t}\left(D_{t-1}^{-1}+c^{-1}\mi\right)}{1+\vxti{t}\left(D_{t-1}^{-1}+c^{-1}\mi\right)\vxi{t}}~,
\label{identity2}
\end{align}

Finally, from \eqref{identity1} we get,
\begin{align*}
&\left(\mi+c^{-1}D_{t-1}\right)^{-1}D_{t}^{-1}\vxi{t}\vxti{t}D_{t}^{-1}\left(\mi+c^{-1}D_{t-1}\right)^{-1}-D_{t-1}^{-1}\\
&\quad+\left(\mi+c^{-1}D_{t-1}\right)^{-1}\left[D_{t}^{-1}\left(\mi+c^{-1}D_{t-1}\right)^{-1}+c^{-1}\mi\right]\\
=~&
 \left(\mi+c^{-1}D_{t-1}\right)^{-1}D_{t}^{-1}\vxi{t}\vxti{t}D_{t}^{-1}\left(\mi+c^{-1}D_{t-1}\right)^{-1}\\
&\quad-D_{t-1}^{-1}
+\Bigg[
\mi-c^{-1}\left(D_{t-1}^{-1}+c^{-1}\mi\right)^{-1}
\Bigg]
\Bigg[
D_{t-1}^{-1}+c^{-1}\mi\vphantom{\frac{
\left(
D_{t-1}^{-1}+c^{-1}\mi
\right)
\vxi{t}\vxti{t}
\left(
D_{t-1}^{-1}+c^{-1}\mi
\right)
}{1+\vxti{t}
\left(
D_{t-1}^{-1}+c^{-1}\mi
\right
)\vxi{t}}}
\\
&\quad-\frac{\left(D_{t-1}^{-1}+c^{-1}\mi\right)\vxi{t}\vxti{t}D_{t-1}^{-1}}{1+\vxti{t}\left(D_{t-1}^{-1}+c^{-1}\mi\right)\vxi{t}}\Bigg]~.
\end{align*}
We further develop the last equality and use \eqref{identity1} and
\eqref{identity2} in the second equality below,
\begin{align*}
 =~ &
 \left(\mi\!+\!c^{-1}D_{t-1}\right)^{-1}\!D_{t}^{-1}\vxi{t}\vxti{t}D_{t}^{-1}\left(\mi\!+\!c^{-1}
 D_{t-1}\right)^{-1}\!-\!D_{t-1}^{-1}\\
&+D_{t-1}^{-1}-\frac{D_{t-1}^{-1}\vxi{t}\vxti{t}D_{t-1}^{-1}}{1+\vxti{t}
\left(D_{t-1}^{-1}+c^{-1}\mi\right)\vxi{t}}\\
 =~ &
 \left[
D_{t-1}^{-1}-\frac{D_{t-1}^{-1}\vxi{t}\vxti{t}\left(D_{t-1}^{-1}+c^{-1}\mi\right)}
 {1+\vxti{t}\left(D_{t-1}^{-1}+c^{-1}\mi\right)\vxi{t}}\right]\vxi{t}\vxti{t}\\
& \Bigg[D_{t-1}^{-1}\vphantom{\frac{\left(D_{t-1}^{-1}+c^{-1}\mi\right)\vxi{t}\vxti{t}D_{t-1}^{-1}}
 {1+\vxti{t}\left(D_{t-1}^{-1}+c^{-1}\mi\right)\vxi{t}}}
-\frac{\left(D_{t-1}^{-1}+c^{-1}\mi\right)\vxi{t}\vxti{t}D_{t-1}^{-1}}
 {1+\vxti{t}\left(D_{t-1}^{-1}+c^{-1}\mi\right)\vxi{t}}
\Bigg]\\
&
-\frac{D_{t-1}^{-1}\vxi{t}\vxti{t}D_{t-1}^{-1}}{1+\vxti{t}\left(D_{t-1}^{-1}+c^{-1}
\mi\right)\vxi{t}}\\
 =~ & -\frac{\vxti{t}\left(D_{t-1}^{-1}+c^{-1}\mi\right)\vxi{t}D_{t-1}^{-1}\vxi{t}\vxti{t}
 D_{t-1}^{-1}}{\left(1+\vxti{t}\left(D_{t-1}^{-1}+c^{-1}\mi\right)\vxi{t}\right)^{2}}~~\preceq~~0 ~.
\end{align*}
 \end{proof}

\subsection{Derivations for \thmref{thm:basic_bound}}
\label{algebraic_manipulation}
\begin{eqnarray*}
 &  & \left(y_{t}-\hat{y}_{t}\right)^{2}+\min_{\vui{1},\ldots,\vui{t-1}}Q_{t-1}\left(\vui{1},\ldots,\vui{t-1}\right)\\
 &  &-\min_{\vui{1},\ldots,\vui{t}}Q_{t}\left(\vui{1},\ldots,\vui{t}\right)\\
 &=  & \left(y_{t}-\hat{y}_{t}\right)^{2}-\veti{t-1}D_{t-1}^{-1}\vei{t-1}+f_{t-1}+\veti{t}D_{t}^{-1}\vei{t}-f_{t}\\
 & = & \left(y_{t}-\hat{y}_{t}\right)^{2}-\veti{t-1}D_{t-1}^{-1}\vei{t-1}\\
 &  &
 +\left(\left(\mi+c^{-1}D_{t-1}\right)^{-1}\vei{t-1}+y_{t}\vxi{t}\right)^{\top}D_{t}^{-1}\\
&&~~~\left(\left(\mi+c^{-1}D_{t-1}\right)^{-1}\vphantom{\left(\mi+c^{-1}D_{t-1}\right)^{-1}}\vei{t-1}+y_{t}\vxi{t}\right)\\
&&
+\veti{t-1}\left(c\mi+D_{t-1}\right)^{-1}\vei{t-1}-y_{t}^{2} ~,
\end{eqnarray*}
where the last equality follows from \eqref{e} and \eqref{f}. We proceed to develop the
last equality,
\begin{eqnarray*}
&=  & \left(y_{t}-\hat{y}_{t}\right)^{2}-\veti{t-1}D_{t-1}^{-1}\vei{t-1}\\
 &  & +\veti{t-1}\left(\mi+c^{-1}D_{t-1}\right)^{-1}D_{t}^{-1}\left
 (\mi+c^{-1}D_{t-1}\right)^{-1}\vei{t-1}\\
 &  & +2y_{t}\vxti{t}D_{t}^{-1}
 \left(\mi+c^{-1}D_{t-1}\right)^{-1}\vei{t-1}\\
 &  &     +y_{t}^{2}\vxti{t}D_{t}^{-1}\vxi{t}
  +\veti{t-1}\left(c\mi+D_{t-1}\right)^{-1}\vei{t-1}-y_{t}^{2}\\
 & = & \left(y_{t}-\hat{y}_{t}\right)^{2}
+\veti{t-1}
\Bigg(-D_{t-1}^{-1}\\
&&
+\left(\mi+c^{-1}D_{t-1}\right)^{-1}
 D_{t}^{-1}
\Big(\mi +c^{-1}D_{t-1}
\Big)^{-1}\\
&&+c^{-1}\left(\mi+c^{-1}D_{t-1}\right)^{-1}
\Bigg)
 \vei{t-1}\\
 &  & +2y_{t}\vxti{t}D_{t}^{-1}\left(\mi\!+\!c^{-1}D_{t-1}\right)^{-1}\vei{t-1} \!+\!y_{t}^{2}\vxti{t}D_{t}^{-1}\vxi{t}\!-\!y_{t}^{2}\\
 &  =&
 \left(y_{t}-\hat{y}_{t}\right)^{2}+\veti{t-1}
\Bigg(-D_{t-1}^{-1}+
\Big(\mi\vphantom{\left[D_{t}^{-1}\left(\mi+c^{-1}D_{t-1}\right)^{-1}+c^{-1}\mi\right]}
+c^{-1}D_{t-1}
\Big)^{-1}
\\
 &  &\left[D_{t}^{-1}\left(\mi+c^{-1}D_{t-1}\right)^{-1}+c^{-1}\mi\right]
\Bigg)\vei{t-1}\\
 &  & +2y_{t}\vxti{t}D_{t}^{-1}\left(\mi\!+\!c^{-1}D_{t-1}\right)^{-1}\vei{t-1}\!+\!y_{t}^{2}\vxti{t}D_{t}^{-1}\vxi{t}\!-\!y_{t}^{2}~.
\end{eqnarray*}

\subsection{Proof of \corref{cor:main_corollary}}
\label{proof_cor_main}
\begin{proof}
Plugging \lemref{lem:bound_1} in \thmref{thm:basic_bound} we have for
all $(\vui{1} \dots \vui{T})$,
\begin{align}
L_T(\textrm{LASER})\nonumber\\
 \leq~&  b\left\Vert
    \vui{1}\right\Vert ^{2}+c V^{(2)}+L_T(\{\vui{t}\})+Y^{2} \ln\left|\frac{1}{b}D_{T}\right|\nonumber\\
&+c^{-1}Y^2 \sum_{t=1}^{T} \tr\paren{D_{t-1}} \nonumber\\
\leq~ &
  b\left\Vert
     \vui{1}\right\Vert ^{2}+L_T(\{\vui{t}\})+Y^{2}
   \ln\left|\frac{1}{b}D_{T}\right|\nonumber\\
&+c^{-1}Y^2\tr\paren{D_0}+c V^{(2)}\nonumber\\
& +c^{-1}Y^2 T d  \max\braces{ \!\!\frac{3X^2 \!+\!
   \sqrt{X^4\!+\!4X^2 c}}{2},b\!+\!X^2\!\!}~,\nonumber
\end{align}
where the last inequality follows from \lemref{eigen_values_lemma}.
The term $c^{-1}Y^2\tr\paren{D_0}$ does not depend on $T$, because
\[
c^{-1}Y^2\tr\paren{D_0}=c^{-1}Y^2d\frac{bc}{c-b}=\frac{\varepsilon}{1-\varepsilon}Y^{2}d~.
\]
To show \eqref{bound1}, note that
\[
V^{(2)} \leq T \frac{\sqrt{2}Y^2dX}{\mu^{3/2}} \Leftrightarrow \mu \leq  \paren{\frac{\sqrt{2}Y^2dXT}{V^{(2)}}}^{2/3}=c~.
\]
We thus have that the right term of \eqref{final_cor} is upper bounded,
\begin{align*}
&\max\braces{  \frac{ 3X^2 +
   \sqrt{ X^4+4X^2 c } }{ 2 },b+X^2 }\\
 \leq~&  \max\braces{  \frac{ 3X^2 +
   \sqrt{ 8X^2 c } }{ 2 },b+X^2 }\\
    \leq~&
\max\braces{  \sqrt{ 8X^2 c },b+X^2 }
\leq 2X\sqrt{ 2c } ~.
\end{align*}
Using this bound and plugging the value of $c$ from \eqref{c1} we bound
\eqref{final_cor},
\begin{align*}
&\paren{\frac{\sqrt{2}T Y^2 d X}{V^{(2)}}}^{2/3} V^{(2)} + Y^2 T d 2X
\sqrt{2 \paren{\frac{\sqrt{2}T Y^2 d X}{V^{(2)}}}^{-2/3}}\\
& \quad=
3\paren{\sqrt{2}T Y^2 d X}^{2/3} \paren{V^{(2)}}^{1/3} ~,
\end{align*}
which concludes the proof.
\end{proof}

\subsection{Details for the bound \eqref{high_drift}}
\label{details_for_second_bound}
To show the bound \eqref{high_drift}, note that,
\[
V^{(2)} \geq T \frac{Y^2dM}{\mu^{2}} \Leftrightarrow \mu \geq \sqrt{ \frac{TY^2dM}{V^{(2)}}}=c~.
\]
We thus have that the right term of \eqref{final_cor} is upper bounded
as follows,
\begin{align*}
&\max\braces{ \frac{3X^2 +
   \sqrt{X^4+4X^2 c}}{2},b+X^2}\\
\leq~ &
\max\braces{ 3X^2,\sqrt{X^4+4X^2 c},b+X^2} \\
\leq~ &
\max\braces{ 3X^2,\sqrt{2}X^2 ,\sqrt{8X^2 c},b+X^2}\\
=~& \sqrt{8X^2}
\max\braces{ \frac{3X^2}{\sqrt{8X^2}},\sqrt{
    c},\frac{b+X^2}{\sqrt{8X^2}}} \\
=~& \sqrt{8X^2}
\sqrt{\max\braces{ \frac{(3X^2)^2}{{8X^2}},
    c,\frac{\paren{b+X^2}^2}{{8X^2}}}}\\
=~& \sqrt{8X^2}
\sqrt{\max\braces{ \mu,c}}
\leq \sqrt{8X^2} \sqrt{\mu} &= M ~.
\end{align*}
Using this bound and plugging $c= \sqrt{{Y^2dMT}/{V^{(2)}}}$ we bound
\eqref{final_cor},
\(
\sqrt{\frac{Y^2dMT}{V^{(2)}}}V^{(2)} + \frac{1}{\sqrt{\frac{Y^2dMT}{V^{(2)}}}} TdY^2M =
2\sqrt{Y^2dMTV^{(2)}} ~.
\)